\documentclass[11pt]{article}

\usepackage[letterpaper,margin=1in]{geometry}
\usepackage{microtype}

\usepackage{amsmath}
\usepackage{amsfonts}
\usepackage{amsthm}

\usepackage{graphicx}
\usepackage{float}
\usepackage[dvipsnames]{xcolor}
\usepackage[normalem]{ulem}
\usepackage{booktabs}
\usepackage{wrapfig}
\usepackage{placeins}
\usepackage{multirow}

\usepackage{algorithm}
\usepackage{algpseudocode}
\algtext*{EndFor}

\usepackage[numbers,sort&compress]{natbib}

\usepackage[hidelinks]{hyperref}
\hypersetup{
    colorlinks,
    linkcolor={red!50!black},
    citecolor={blue!50!black},
    urlcolor={blue!80!black}
}

\usepackage{graphicx}
\usepackage{amsmath}
\usepackage{amsfonts}
\usepackage{amsthm}
\usepackage{algorithm}
\usepackage{algpseudocode}
\algtext*{EndFor}
\usepackage{float}
\usepackage[dvipsnames]{xcolor}
\usepackage[normalem]{ulem}
\usepackage{booktabs} 
\usepackage{wrapfig}
\usepackage{placeins}
\usepackage{multirow}

\usepackage[hidelinks]{hyperref}
\hypersetup{
    colorlinks,
    linkcolor={red!50!black},
    citecolor={blue!50!black},
    urlcolor={blue!80!black}
}

\newtheorem{corollary}{Corollary}

\newtheorem{theorem}{Theorem}
\newtheorem{assumption}{Assumption}

\title{When Metropolis and Hastings Meet Bradley and Terry:

Exact MCMC From Preference Voting}

\author{
Ariel Smogorghevski\textsuperscript{1}
\quad
Nir Rosenfeld\textsuperscript{2}
\quad
Yaniv Romano\textsuperscript{2,3}
\\[0.8em]
\small
\textsuperscript{1}Department of Data and Decision Sciences,
Technion -- Israel Institute of Technology\\
\small
\textsuperscript{2}Department of Computer Science,
Technion -- Israel Institute of Technology\\
\small
\textsuperscript{3}Department of Electrical and Computer Engineering,
Technion -- Israel Institute of Technology
}

\date{}

\begin{document}

\maketitle

\begin{abstract}

Sampling from distributions conditioned on desired semantic properties is an emerging challenge in modern generative modeling. Metropolis--Hastings (MH) provides a principled route to conditional sampling, but requires access
to exact pointwise target-density evaluations, which are not available in generative settings.
Meanwhile, \textit{pairwise comparisons} by humans or model ``judges''
are highly accessible and have proved valuable across diverse applications.
We introduce \texttt{Pref-MH}, a general exact MH sampler for judge-induced conditional distributions using only stochastic binary pairwise comparisons.
Our key observation is that the MH unnormalized density ratio
matches the preference odds of the Bradley--Terry (BT) choice model.
The central challenge is that while MH requires precise ratio computation,
BT judges provide only sampled binary feedback.
To this end,
we develop a valid accept/reject rule 
whose resulting Markov chain provably converges to the target distribution. 
We further show that, for a fixed proposal kernel and budget, \texttt{Pref-MH} is optimal in the Peskun--Tierney sense among this class of exact reversible acceptance rules. Experiments on text generation and molecular design with LLM judges, as well as image generation with VLM judges, demonstrate that \texttt{Pref-MH} provides a practical and flexible approach to conditional sampling when comparative feedback is relatively easy to obtain.
\looseness=-1

\end{abstract}

\begingroup
\renewcommand{\thefootnote}{}
\footnotetext{Code is available online at \url{https://github.com/Ariels34/Pref-MH}.}
\addtocounter{footnote}{-1}
\endgroup

\begin{figure}[h]
    \centering
    \includegraphics[width=1.0\textwidth]{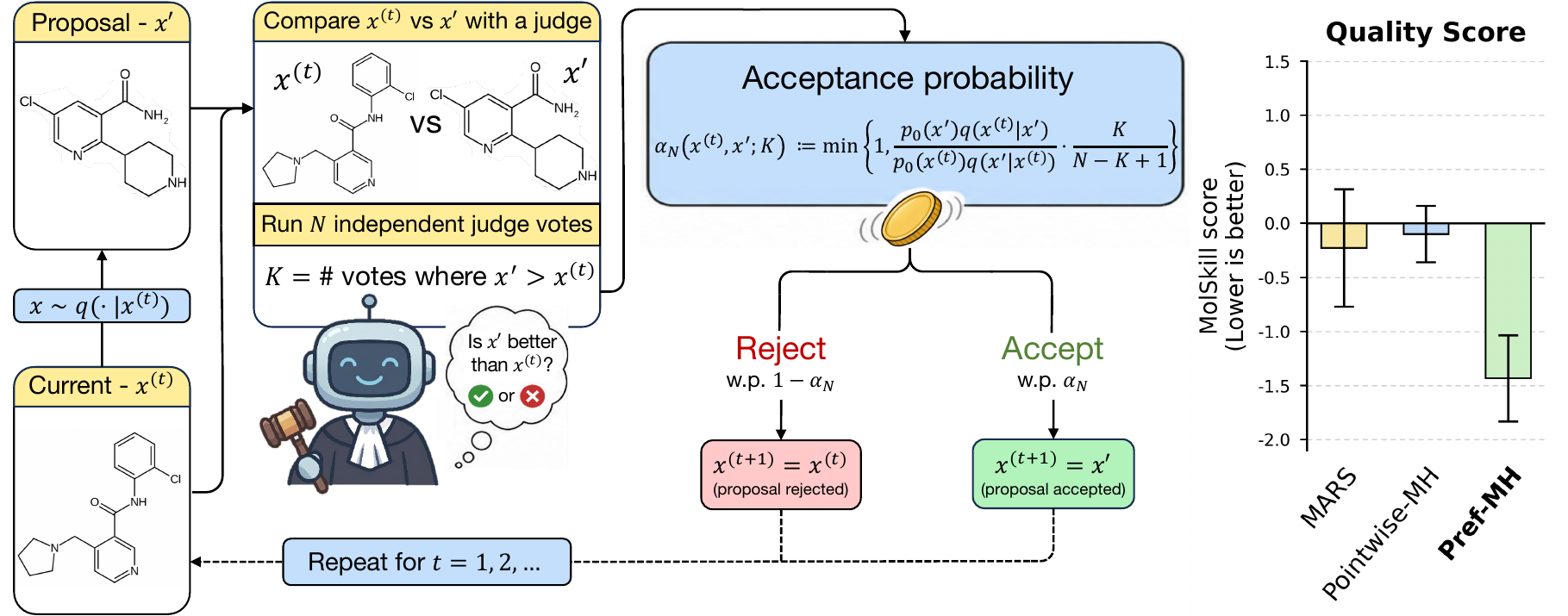}
    \caption{
    \textbf{Illustration of our proposed \texttt{Pref-MH} for molecular design for drug discovery.}
    At each step $t$: (i) a proposal sample $x'$ is generated given the current state $x^{(t)}$; (ii) $x'$ is then compared against $x^{(t)}$ by using $N$ stochastic judge votes; (iii) the proposal $x'$ is either accepted or rejected according to our vote-based acceptance rule $\alpha_N$. The right panel summarizes the results of a molecular-design experiment (Section~\ref{exp:molecular-design}), highlighting that \texttt{Pref-MH} outperforms the score-based baselines (lower is better).
    }
    \label{fig:method_illustration}
\end{figure}

\section{Introduction}

Drawing samples from conditional distributions is a fundamental challenge in computational statistics and machine learning. Markov Chain Monte Carlo (MCMC) methods provide a canonical framework for this task, with
Metropolis--Hastings (MH) serving as one of their most widely used instances
\citep{metropolis1953equation, hastings1970monte}. Such methods have found widespread use
in probabilistic machine learning and across diverse applications ranging from statistical physics to computational biology \citep{zhou2009bayesian, hansmann1999new, huelsenbeck2001bayesian, gill2013improving, tu2002image, turner2019metropolis}. Unlike methods that return a single ``best'' solution, MCMC enables repeated sampling of plausible instances from a target distribution. This makes it especially well suited to tasks requiring multiple diverse outputs, or settings in which a human selects a preferred option from a set of candidates. For example, in molecular design, MCMC can generate multiple molecules that satisfy a given criterion yet differ in their trade-offs across key properties such as efficacy, toxicity, and synthesizability \citep{fromer2023computer, a2012multi, xie2021mars, fu2021mimosa}.
\looseness=-1

Modern generative models, such as large language models (LLMs) and vision-language models (VLMs), 
offer the potential to unlock a wide range of new applications that require sampling from semantically meaningful conditional distributions over complex modalities \citep{hu2017toward,dathathri2020plug,yang2021fudge,krause2021gedi,liang2024controllable}.
However, while sampling from generative models in itself is straightforward,
it is unclear how to rigorously ensure that instances are indeed drawn from the correct and full target conditional distribution.
For example, if we wish to generate 
responses that are `helpful', `polite',
and in `Shakespearean style', then prompting for these properties does not guarantee that generated completed sequences truly satisfy them,
nor does it provide control over the nature of variation across samples.

Our work addresses this gap by introducing a novel Metropolis–Hastings (MH) algorithm tailored to generative models that enables principled conditional sampling. MH algorithms work by repeatedly proposing a candidate sample given the current sample,
which is either accepted or rejected according to a predefined acceptance rule designed to ensure convergence to the desired conditional distribution \citep{hastings1970monte,chib1995understanding,tierney1994markov}. The key methodological challenge is that implementing the acceptance rule requires access to the true unnormalized density ratios of
the target conditional distribution. For current generative models, this is typically infeasible \citep{dandapanthula2026reallytiltingmechanicsreward}: informally, in our example, this would require computing the probability of any given text under the distribution of all texts that are `helpful', `polite', and `Shakespearean'. Similarly, in molecular design, it would require assigning exact conditional likelihoods to molecules satisfying complex and often competing properties such as efficacy, toxicity, and synthesizability.

To overcome this challenge, we leverage the fact that pairwise comparisons have become ubiquitous in generative modeling,
whether in model alignment via human comparative feedback \citep{christiano2017deep,stiennon2020learning,ouyang2022training,bai2022training,rafailov2023direct}
or in the use of generative models as ``judges'' for benchmarks and competitions \citep{zheng2023judging}.
Such binary pairwise comparisons are readily obtainable, broadly applicable across tasks, and have often proven effective in practice.
However, using pairwise comparisons for conditional sampling raises two difficulties. First, the canonical MH algorithm innately depends on access
to \emph{pointwise} density ratios, making it unclear how to incorporate \emph{pairwise} feedback. Second, it requires computing density ratios precisely in order to guarantee convergence to the target distribution,
whereas judges provide only sample access to stochastic binary votes between alternatives.

This structural limitation raises a fundamental question: \emph{how to formulate an exact MCMC sampler when the target distribution is accessible only through stochastic, binary pairwise comparisons?} 
Exactness here means that the resulting chain's stationary distribution matches the desired target judge-induced conditional distribution.
In this paper, we show that such exact sampling is indeed possible by introducing \texttt{Pref-MH}: the first exact MH framework that uses only binary pairwise preference comparisons, with support for general state spaces and arbitrary proposal kernels. Figure~\ref{fig:method_illustration} visualizes the method in the context of de novo molecular design, where the goal is to generate molecules (SMILES-encoded) for downstream drug discovery campaigns; a task we revisit later in our experiments.

\subsection*{Preview of \texttt{Pref-MH} and key contributions}
We begin by observing a simple but powerful connection (Section~\ref{sec:BT-integration}): under the assumption that the conditional target is structured as a Gibbs distribution, the MH acceptance ratio closely mirrors the Bradley-Terry (BT) preference model from discrete choice \citep{bradley1952rank, luce1959individual}.
In MH, the acceptance rule depends on the ratio between the target densities of two candidate inputs. Under the BT model, the odds of comparison probabilities equal the ratio between the judge's implicit (and unobserved) preference `strength' for the two candidates.
We link the two by establishing a theoretical equivalence between both ratios.

Our key contribution is to then turn this connection between MH and BT into a practical, exact sampling algorithm. The main difficulty is that the sampler does not have access to the underlying comparison probability itself; rather, it can only query judges to obtain binary preference votes for a given input pair. 
A natural idea is therefore to estimate the required odds ratio 
by first estimating each comparison probability independently from multiple samples,
and then plugging the ratio of these estimates into the standard MH acceptance rule. We prove a strong impossibility result: \emph{any} procedure that uses a fixed number of binary comparisons---including such plug-in estimation---%
cannot implement the oracle acceptance rule (Thm.~\ref{thm:no-fixed-budget-mh}).
Intuitively, this holds since even small errors in estimating the acceptance ratio can accumulate to completely distort the chain's induced stationary distribution \citep{alquier2016noisy,mitrophanov2005sensitivity}. We circumvent this impossibility by designing a randomized accept/reject mechanism whose marginal acceptance probabilities satisfy the Hastings ratio induced by the pairwise judges; see the illustration in Figure~\ref{fig:method_illustration}. 
Although this does not implement the oracle rule,
it nonetheless yields a Markov chain whose stationary distribution
provably matches the target conditional distribution---%
despite using only a fixed number of sampled binary comparisons at each transition 
(Thm.~\ref{thm:n-vote-exact} and Cor.~\ref{cor:pref-mh-conv}).

Our resulting \texttt{Pref-MH} method is an exact reversible MCMC sampler that never evaluates a pointwise score function and never plugs in an estimated comparison probability or numerical approximation to the acceptance ratio.
Moreover, we establish a strong optimality result (Thm.~\ref{thm:peskun-optimality}): for any fixed proposal mechanism and fixed comparison budget, 
\texttt{Pref-MH} is optimal in the Peskun--Tierney
sense among the class of exact competitors using the same number of pairwise comparisons \citep{peskun1973optimum,tierney1998note}.
Because the rule is parameterized by a fixed number of judge queries, the user can choose the per-step comparison budget; increasing this budget can, in turn, support higher acceptance probabilities, allowing the sampler to explore the space of plausible candidates more effectively.

Finally, our accept/reject mechanism extends naturally to multi-judge settings (Sec.~\ref{sec:multi-judge}). This is because our method treats proposal generation and comparative evaluation as separate components: one proposes candidate samples, while the other judges their relative fit to the desired properties. This allows us to combine multiple independent comparative criteria---for example, 
one for helpfulness, one for politeness, and one for Shakespearean-ness--%
without reducing them to a single manually-weighted aggregate
score or a single judge query.
More broadly, this yields a modular way to build samplers from diverse comparative feedback sources, where each ``judge'' specializes in a single property,
enabling conditional sampling that is more direct and easier to adapt across tasks. 
\looseness=-1

We conclude with experiments demonstrating the utility and flexibility of the resulting sampler across several conditional generation tasks, including synthetic experiments, text generation, image generation, and molecular design for drug discovery. Figure~\ref{fig:method_illustration} highlights the results from the molecular-design experiment (Sec.~\ref{exp:molecular-design}), showing that our \texttt{Pref-MH} achieves better performance on an external evaluation metric, than MH baselines that do not use pairwise LLM judgments. The MolSkill metric presented reflects a human expert quality score \citep{choung2023extracting} and we do not use it to guide the sampling, making it an objective performance measure.

\paragraph{Related work.}
\citet{fotakis2022perfect} study perfect sampling
from pairwise comparisons over a fixed finite item set, using comparison samples
from a local sampling scheme. Human-in-the-loop MCMC methods, such as
\citet{sanborn2007markov,sanborn2010uncovering,harrison2020gibbs}, use human
choices to construct sampling procedures for subjective or semantic
representations, but are tied to specific behavioral protocols, choice models,
or update structures. A different line of work uses MH or related MCMC methods
in generative modeling and controllable generation, including energy-based
sampling for masked or controllable language models and quality-aware sampling
for machine translation
\citep{goyal2022exposing,mireshghallah2022mix,forristal2023block,du2024principled,gonzalez2025constrained,faria2024quest}.
Recent inference-time sampling methods have also used MCMC-inspired procedures
to improve LLM reasoning, but rely on base-model likelihoods rather than
pairwise comparison-only access \citep{karan2026reasoning}.
These methods strictly rely on an evaluable energy, score, reward, likelihood,
constraint indicator, or quality metric. Our \texttt{Pref-MH} differs in its technical scope: It offers an MH
accept/reject rule for general state spaces, user-chosen proposal kernels, and
builds on comparisons between the current and proposed states.
\texttt{Pref-MH} is not restricted to finite catalogs, symmetric proposals,
Barker-style updates, or pointwise reward evaluations. Further, \texttt{Pref-MH}
supports any fixed finite comparison budget per transition, trading judge cost for higher acceptance and exploration while preserving the target distribution
exactly. Additional related work is in Appx.~\ref{app:related-work}.
\looseness=-1

\section{Problem Setup and Background}
\label{sec:setup}

Let $\mathcal{X}$ denote the sample space
(e.g., plots for a film, math word problems, posters for an event),
and let $M$ be a property of interest,
such as adherence to an instruction, factual correctness, or a particular style. We assume access to a base distribution $p_0$ over $\mathcal{X}$---typically a generative model---%
but no direct mechanism for conditioning $p_0$ on $M$.
To define this conditioning, we consider settings in which the question of
whether an input $x$ admits the property $M$ can be evaluated by a judge, e.g., human or generative model.
Formally, we define the \emph{judge-induced target conditional distribution} to be:
\looseness=-1
\begin{equation}
\label{eq:target}
    \pi(x) := p(x \mid M)
    = \frac{p_0(x)\,p_J(M \mid x)}{\int_{\mathcal X} p_0(u)\,p_J(M \mid u)\,d\mu(u)}
    \propto p_0(x)\,p_J(M \mid x),
\end{equation}
Here $p_J(M \mid x)$ denotes the probability that $x$ satisfies $M$ according to a given judge $J$. In words, the target $\pi$ can be viewed as a reweighting of the base distribution $p_0$ toward samples which the judge $J$ views as more likely to exhibit the property $M$. Crucially, in this work, we do not assume that we have direct access to $p_J(M \,|\, x)$, but rather have only \emph{sample access to pairwise comparisons} between two candidates $x$ and $x'$. This restriction is in line with many settings in which head-to-head comparisons are available via a judge, whereas the true probability $p_J(M \mid x)$ is generally unknown.
\looseness=-1

In general, $M$ may reflect multiple desiderata that should hold simultaneously, such as helpfulness, politeness, and stylistic fidelity. To ease the exposition of our method, we begin with a single property $M$ and return to the multi-condition case
$M=(M_1,\ldots,M_m)$ in Section~\ref{sec:multi-judge}.

\paragraph{Metropolis--Hastings sampling.}
MH  is a general algorithm that produces samples from a target distribution $\pi$ by constructing a Markov chain whose stationary distribution matches the target.
Given an initial state $x_0$
and a chosen \emph{proposal kernel} $q(x' \,|\, x)$,
MH generates a sequence of samples in the following manner:
At each state $x$, a new candidate state $x'$ is proposed by sampling $x' \sim q(\cdot \mid x)$.
Then, $x'$ is accepted with probability:
\begin{equation}
\label{eq:mh-acceptance}
\alpha_{\mathrm{MH}}(x,x')
=
\min\left\{1,\,
\frac{\pi(x')\,q(x \mid x')}{\pi(x)\,q(x' \mid x)}
\right\}.
\end{equation}
If successful, the state transitions from $x$ to $x'$,
otherwise it remains at $x$.
The proposal kernel $q$ is the main design choice controlling how the sampler explores $\mathcal X$, from global proposals such as resampling from $p_0$ to local proposals such as modifying the current state. A key feature of our setup is that it leaves this proposal mechanism user-specified. Thus, \texttt{Pref-MH} could be combined with any valid proposal kernel, subject to the usual regularity conditions needed for convergence as stated in Corollary~\ref{cor:pref-mh-conv}. 
Repeating this accept/reject step iteratively produces the desired Markov chain.
\looseness=-1

In our setting, directly evaluating the target density $\pi$ is infeasible.
However, MH only requires ratios of target densities, so substituting Eq.~\eqref{eq:target} into the MH ratio gives:
\begin{equation}
\label{eq:mh-ratio-conditional}
\frac{\pi(x')\,q(x \mid x')}{\pi(x)\,q(x' \mid x)}
=
\frac{p_0(x')\,q(x \mid x')}{p_0(x)\,q(x' \mid x)}
\cdot
\frac{p_J(M \mid x')}{p_J(M \mid x)}.
\end{equation}
The first factor is tractable since $p_0$ is assumed given and $q$ is a design choice.
Since we do not assume direct access to $p_J(M \,|\, x)$, the remaining
challenge lies in computing the likelihood ratio
$p_J(M \,|\, x')/p_J(M \,|\, x)$.

\section{Proposed Method: Exact Sampling from Preference-Based Judges}
\label{sec:exact-sampling}

\subsection{Integrating Bradley--Terry within Metropolis--Hastings}
\label{sec:BT-integration}
To present the basic construction of our method, suppose for a moment that we have access to $p_J(M \,|\, x)$. A straightforward way to compute the likelihood ratio in Eq.~\eqref{eq:mh-ratio-conditional} is to compute $p_J(M\,|\,x)$ and $p_J(M\,|\,x')$ independently and then divide.
Unfortunately, this is not rigorously possible for generative models over arbitrary properties $M$.
Instead,
we make the (implicit) parametric assumption
that $p_J(M \,|\, x) \propto \exp({s(x)})$ for some \emph{unknown}, latent score function $s:\mathcal{X}\to\mathbb{R}$. Under this construction, the target in Eq.~\eqref{eq:target} admits the following Gibbs, score-tilted form
\begin{equation}
\label{eq:tilted_pi}
    \pi(x) \propto p_0(x)\,e^{s(x)}.
\end{equation}
Tilted distributions of this form play a central role in modern generative modeling and LLM alignment, where one seeks to bias sampling toward a desired reward while staying close to the base distribution \citep{christiano2017deep, ouyang2022training, rafailov2023direct, levine2018reinforcement}.

While the latent score is unknown, we can still leverage any pairwise comparison ``judge''
whose (stochastic) preferences over elements in $\mathcal{X}$ are induced by $s$.
For a judge $J$ of the property $M$,
the parametric assumption on $s$ entails
probabilistic pairwise comparisons that follow:
\begin{equation}
\label{eq:BT-model}
    p_J(x \prec x')
    =
    \sigma\!\big(s(x')-s(x)\big),
    \qquad
    \sigma(t)=\frac{1}{1+e^{-t}}.
\end{equation}
where $p_J(x \prec x')$ denotes the probability that judge $J$
prefers $x'$ to $x$.
This corresponds to the canonical BT model for discrete choice \citep{bradley1952rank, luce1959individual}.

Substituting Eq.~\eqref{eq:BT-model} into Eq.~\eqref{eq:mh-ratio-conditional},
we can now write the likelihood ratio 
using pairwise preferences:
\begin{equation}
\label{eq:bt-odds-warmup}
\frac{p_J(M\mid x')}{p_J(M\mid x)}
=
e^{s(x')-s(x)}
=
\frac{p_J(x \prec x')}{p_J(x \succ x')}
=
\frac{p_J(x \prec x')}{1-p_J(x \prec x')},
\end{equation}

known as the \emph{win-rate ratio}.
Plugging this into Eq.~\eqref{eq:mh-ratio-conditional} yields a 
preference-based acceptance rule:
\looseness=-1
\begin{equation}
\label{eq:mh-acceptance-bt}
\alpha_{\mathrm{BT\text{-}MH}}(x,x')
:=
\min
\bigg\{ 
1,\,
\underbrace{\frac{p_0(x')\,q(x \mid x')}{p_0(x)\,q(x' \mid x)}}_{r_0(x,x')}
\cdot
\underbrace{
\frac{p_J(x \prec x')}{1-p_J(x \prec x')}
}_{\text{win-rate ratio $w(x,x')$}}
\bigg\}. 
\end{equation}
The above derivation shows that, if the comparison probability were available,
the BT model would provide the likelihood ratio needed to run MH without
knowing the latent score itself.
Essentially, this enables us to replace the pointwise question ``does $x$ satisfy $M$?'' with the pairwise question ``which of $x$ and $x'$ is more likely $M$?''.
However, this creates a different challenge:
we need access to the win-rate ratio \(w(x,x')\). As a result,
Eq.~\eqref{eq:mh-acceptance-bt} does not yet provide an implementable sampler,
since the sampler observes only binary comparisons drawn by the judge.
The question, then, is how to construct an exact reversible kernel for \(\pi\)
given only access to such one-bit judge outcomes.

\subsection{Impossibility of directly implementing the BT-MH acceptance rule}

Since each judge's query returns only a binary vote, one may hope that a fixed
number of such outcomes is enough to implement the oracle acceptance
probability $\alpha_{\mathrm{BT\text{-}MH}}(x,x')$. For example, one natural attempt is to query the judge \(N\) times on the pair \((x,x')\), estimate \(p_J(x\prec x')\), and plug this estimate into Eq.~\eqref{eq:mh-acceptance-bt}. However, we show that this strategy, and in fact \emph{any} fixed-budget procedure over binary votes, cannot reproduce the oracle BT-MH acceptance probability for every value of the unknown comparison probability $p_J(x\prec x') = p \in(0,1)$.

\begin{theorem}[No fixed-budget implementation of the BT-MH acceptance rule]
\label{thm:no-fixed-budget-mh}
Fix a proposed move \(x\to x'\), and suppose that $r_0(x, x') \in (0, \infty)$.
For any \(N<\infty\), there is no
procedure using only $N$ samples
$J^{(1)},\ldots,J^{(N)} \sim \mathrm{Bernoulli}(p)$
whose marginal acceptance probability equals $\alpha_{\mathrm{BT\text{-}MH}}(x,x')$
simultaneously for all $p \in(0,1)$. 

The proof is given in Appx. \ref{app:no-fixed-budget-mh}
\end{theorem} 
A few remarks are in place. First, the assumption \(r_0(x,x')\in(0,\infty)\) merely excludes degenerate proposed moves whose acceptance decision is independent of the judge outcomes. Second, the fixed-budget requirement is important in modern LLM- or human-as-a-judge applications in which judge queries are often costly, and a sampler should have a prescribed finite comparison cost per MH step. Third, the theorem does not rule out exact fixed-budget sampling altogether; rather, it rules out any attempt for reproducing the exact oracle BT-MH acceptance probability from finitely many binary comparisons. In turn, this highlights a fundamental challenge in obtaining an exact \texttt{Pref-MH} sampler. This motivates the next subsection, where we construct a valid accept/reject rule directly.
\looseness=-1

\subsection{A budget-optimal acceptance rule via $N$-vote construction}
\label{sec:acceptance-design}

We now give an explicit fixed-budget acceptance rule for \texttt{Pref-MH}---%
using only binary judge outcomes, while ensuring that the resulting Markov chain exactly matches the target distribution $\pi$.

Fix a current state $x$ and a proposal $x' \sim q(\cdot\mid x)$. Our construction queries the judge independently $N$ times on the pair $(x,x')$, resulting in $ 
J^{(1)},\ldots,J^{(N)} \overset{\mathrm{iid}}{\sim} \mathrm{Bernoulli}(p_J(x \prec x')).$ We denote by $K := \sum_{t=1}^N J^{(t)},$
the number of votes favoring the proposal $x'$. With this in place, we now define our \texttt{Pref-MH} acceptance rule, given by
\begin{equation}
\label{eq:n-vote-rule}
\alpha_N(x,x';K)
:=
\min\!\left\{1,\;
r_0(x,x')\,\frac{K}{N-K+1}
\right\}.
\end{equation}

A useful way to view the above $N$-vote rule is as follows. We start from the
tractable baseline factor $r_0(x,x')$ and then adjust it according to the
evidence from the judge. The vote count $K$ summarizes that evidence. If many
votes favor the proposal $x'$, it is desired to upweight $r_0(x,x')$ and increase the probability of accepting $x'$. By contrast, if few votes favor the proposal, it is desired to reduce the acceptance probability by downweighting $r_0(x,x')$. The factor above does exactly this. When $K=N$, the factor equals $N$, giving the largest upweighting;
when $K=0$, the factor equals $0$, ensuring the proposal will be rejected; and when the vote
is roughly split, the factor is close to $1$, so the decision is driven mainly by
$r_0(x,x')$.

This interpretation is also consistent with the BT odds. Indeed,
for large $N$,
\[
\frac{K}{N-K+1}
=
\frac{K/N}{1-K/N+1/N}
\xrightarrow[N\to\infty]{\mathrm{a.s.}}
\frac{p_J(x \prec x')}{1-p_J(x \prec x')},
\]
so the vote-based factor behaves like a finite-sample surrogate for the ideal
win-rate ratio. In turn, since the acceptance probability is a continuous function of this vote-based
factor, the realized acceptance rule converges almost surely to the ideal
MH update:
\[
\min\!\left\{1,\; r_0(x,x')\,\frac{K}{N-K+1}\right\}
\xrightarrow[N\to\infty]{\mathrm{a.s.}}
\min\!\left\{1,\; r_0(x,x')\,\frac{p_J(x \prec x')}{1-p_J(x \prec x')}\right\}.
\]
Hence, our proposed $N$-vote family may be viewed as a fixed-budget exact approximation toward the full
MH acceptance rule, with larger $N$ yielding behavior closer to the ideal update.

In view of the impossibility result, the key question is whether the $N$-vote rule in Eq.~\eqref{eq:n-vote-rule} still preserves the desired target distribution
for every fixed finite $N$. The next theorem shows that this is indeed the case. The key idea behind the proof is to show that the $N$-vote rule satisfies the detailed balance property, defined as follows:
\begin{equation}
\label{eq:detailed-balance}
\pi(x)\,
q(x'\mid x)\,
\bar\alpha_N(x,x')
=
\pi(x')\,
q(x\mid x')\,
\bar\alpha_N(x',x),
\ \ \text{for every} \ \
x\neq x'.
\end{equation}
Above, $\bar\alpha_N(x,x')$ denotes the acceptance probability, averaged over the random judge outcomes; it is formally defined in Eq.~\eqref{eq:alpha-bar} in the appendix.

\begin{theorem}[Exactness of the $N$-vote acceptance rule]
\label{thm:n-vote-exact}
Under the BT parametric assumption from Section~\ref{sec:BT-integration}, for every integer $N \ge 1$, the $N$-vote acceptance rule in
Eq.~\eqref{eq:n-vote-rule}, combined with a valid proposal kernel $q$, defines a
Markov transition kernel that satisfies detailed balance with respect to
$\pi$. Consequently, the resulting chain is reversible and has $\pi$ as a
stationary distribution.

The proof is given in Appx. \ref{app:n-vote-proof}
\end{theorem}

Having converted binary comparison outcomes into a valid, implementable accept/reject rule, we now summarize the resulting single-judge \texttt{Pref-MH} sampler in Algorithm~\ref{alg:single-judge}.

\begin{algorithm}[h]
\caption{Single-judge Pref-MH with exact $N$-vote acceptance}
\label{alg:single-judge}
\begin{algorithmic}[1]
\Require Initial state $x^{(0)}\in\mathcal X$, proposal kernel $q(\cdot\mid x)$, comparisons $N\ge 1$, iterations $T$
\For{$t=0,1,\ldots,T-1$}
    \State Propose $x' \sim q(\cdot\mid x^{(t)})$
    
    \State Query $J$ independently $N$ times on $(x^{(t)},x')$, and obtain $J^{(1)},\ldots,J^{(N)}.$
    \State Set $K= \sum_{i=1}^N J^{(i)}$, and compute    
    \[
    x^{(t+1)}
    =
    \begin{cases}
    x', & \text{with probability }
    \min\!\left\{1,\;
    \dfrac{p_0(x')\,q(x^{(t)}\mid x')}{p_0(x^{(t)})\,q(x'\mid x^{(t)})}
    \cdot
    \dfrac{K}{N-K+1}
    \right\}, \\[1.2ex]
    x^{(t)}, & \text{otherwise.}
    \end{cases}
    \]
\EndFor
\State \Return the trajectory $x^{(0)},x^{(1)},\ldots,x^{(T)}$.
\end{algorithmic}
\end{algorithm}

As with standard MH, stationarity alone does not by itself imply convergence
from an arbitrary initialization. To achieve this, we impose regularity assumptions on the proposal mechanism, stated formally in Appx.~\ref{app:pref-mh-convergence}. These are common assumptions in the MH literature and are not specific to preference-based sampling \citep{tierney1994markov, roberts2004general}. Informally, these ensure that the chain can explore the support of the target distribution and cannot get trapped in a deterministic cycle. For example, both $q(\cdot \,|\, x) = p_0(\cdot)$ and a $q$ which only resamples a suffix of $x$ using $p_0$ satisfy these assumptions \citep{liu1996metropolized,faria2024quest,karan2026reasoning}. 
\looseness=-1

\begin{corollary}[Convergence of the Pref-MH chain]
\label{cor:pref-mh-conv}
Consider the setting of Theorem~\ref{thm:n-vote-exact}, and fix any comparison
budget $N\ge 1$. Under the standard MH regularity assumptions on the proposal
mechanism, stated formally in Appx.~\ref{app:pref-mh-convergence}, the
Markov chain generated by Algorithm~\ref{alg:single-judge} converges in
distribution to the judge-induced target distribution $\pi$, for every
initialization in the support of $\pi$.

The proof is given in Appx. \ref{app:pref-mh-convergence}
\end{corollary}

\subsection{Optimality}
\label{sec:optimality}

We turn to show that our $N$-vote construction is optimal among exact
fixed-budget acceptance rules. The sense of optimality we use is the standard
Peskun--Tierney (PT) ordering for MCMC. In our setting, this ordering has a simple form. The proposal
kernel $q$ is fixed, and competing sampling methods differ only in how they decide
whether to accept or reject a proposal after observing exactly $N$ judge
queries. Thus, it is enough to compare their acceptance probabilities for each
proposed move. We say that rule $A$ dominates rule $B$ in a PT-sense if, for every proposed
move $x\to x'$,
\[
\mathbb P_A(\text{accept }x'\mid x,x'\text{ proposed})
\geq
\mathbb P_B(\text{accept }x'\mid x,x'\text{ proposed}).
\]
Above, the probability is taken with respect to the randomness of the acceptance rule, which is affected not only by the judge but also by the design of $A$ and $B$.
Intuitively, the idea is that among samplers that preserve the same target distribution, the one that accepts more proposals explores the space at least as efficiently.

The next theorem shows that our $N$-vote rule is optimal in this PT sense. That is, there is no other exact acceptance rule that can assign a larger acceptance probability to any proposed move using the same (i) proposal mechanism $q$, (ii) judge $J$, and (iii) number of
judge queries $N$.

\begin{theorem}[Peskun--Tierney optimality of the $N$-vote rule]
\label{thm:peskun-optimality}
Fix a proposal kernel $q$ and a comparison budget $N\geq 1$. Among all exact acceptance rules satisfying detailed balance with respect to $\pi$, that use the same proposal kernel and only the outcomes of exactly $N$ judge
queries per proposal, the $N$-vote rule Peskun--Tierney dominates every
competitor. Equivalently, for every proposed move $x\to x'$, we have 
$\mathbb P_{\text{$N$-vote}}(\textup{accept }x'\mid x,x'\textup{ proposed})
\geq
\mathbb P_{\text{competitor}}(\textup{accept }x'\mid x,x'\textup{ proposed}).
$

The proof is given in Appx. \ref{app:optimality-proof}
\end{theorem}

\subsection{Composing multiple judges}
\label{sec:multi-judge}

We now extend our \texttt{Pref-MH} to settings where one wishes to favor several attributes simultaneously---for example, responses that are 
helpful, polite, and stylistically appropriate---without collapsing
them into a single hand-engineered score or judge.

Let $M_1,\ldots,M_m$ denote the desired conditions, and $J_1,\ldots,J_m$ the associated judges. We consider the composite target 
\begin{equation}
\label{eq:multi-target-main}
\pi_m(x)
:=
p(x\mid M_1,\ldots,M_m)
\propto
p_0(x)\prod_{i=1}^m p_{J_i}(M_i\mid x) \propto p_0(x) \exp \left({\sum_{i=1}^m s_i(x)}\right).
\end{equation}
The first decomposition holds due to Bayes' rule and under conditional independence given $x$, as shown in Appx. \ref{app:bayes_mh_derivation}. This condition is satisfied in our setting because each event $M_i$ is
determined by an independent judge $J_i$, conditional on $x$. The second transition holds under the parametric assumption
that $p_{J_i}(M_i \,|\, x) \propto \exp\left(s_i(x)\right)$, $i=1, \ldots ,m$, similarly to Eq.~\eqref{eq:tilted_pi}. As in the single-judge setting, we assume that judge $J_i$ follows a BT model with latent score $s_i$, satisfying $p_{J_i}(x\prec x')
=
\sigma\!\bigl(s_i(x')-s_i(x)\bigr),
$
where $p_{J_i}(x\prec x')$ denotes the probability that judge $J_i$ prefers
$x'$ to $x$. 

Accordingly, for a proposal kernel $q(x'\mid x)$, the MH
ratio factorizes as
\begin{equation}
\label{eq:multi-ratio-main}
\frac{\pi_m(x')\,q(x\mid x')}{\pi_m(x)\,q(x'\mid x)}
=
r_0(x,x')
\prod_{i=1}^m
\frac{p_{J_i}(x\prec x')}{1-p_{J_i}(x\prec x')}.
\end{equation}
Thus, the single-judge win-rate ratio is replaced by a product of the multiple judges' win-rate ratios. Since these comparison probabilities are not observed directly, for each
judge $J_i$, we query the pair $(x,x')$ independently $N$ times and record in $K_i$ the number of votes favoring $x'$. Finally, our multi-judge Pref-MH $N$-vote acceptance rule is given by
\begin{equation}
\label{eq:multi-n-vote-main}
\alpha_N^{(m)}(x,x';K_1,\ldots,K_m)
:=
\min\!\left\{1,\;
r_0(x,x')\prod_{i=1}^m \frac{K_i}{N-K_i+1}
\right\}.
\end{equation}
Indeed, when $m=1$, the rule in Eq.~\eqref{eq:multi-n-vote-main} reduces to the single-judge
construction. This highlights that the multiplication keeps the construction modular: adding a condition amounts to adding a
judge and multiplying by its corresponding $N$-vote factor.

Algorithm~\ref{alg:multi-judge} (Appx.~\ref{app:multi-judge-alg}) summarizes our resulting multi-judge \texttt{Pref-MH} sampler. The formal statements and proofs of the exact MCMC property and PT optimality are given in Appx.~\ref{app:multi-judge}.

\section{Experiments}

We evaluate \texttt{Pref-MH} in four different settings. The first is a controlled experiment that demonstrates the validity of our construction. The second focuses on multi-condition sampling for open-ended text generation, showing the advantage of using separate judges when conditions are complex. The third addresses image generation under multiple conditions, illustrating the ability of Pref-MH to sample from the target even when the state space is over a complex continuous modality. Finally, we study de novo molecular design, a real-world scientific design task where generating diverse candidates with desirable chemical properties is critical for early-stage drug discovery. An additional machine-translation experiment is presented in Appendix~\ref{app:exp-mt}.

\subsection{Synthetic validation}
\label{exp:synthetic}

To illustrate how \texttt{Pref-MH} converges to the true target distribution
despite inexact pairwise feedback,
we consider a synthetic setting in which the target distribution is known exactly.
We set $p_0(x)$ to be an arbitrary, non-uniform categorical distribution over
$k\in\{0,\ldots,240\}$.
We then set $p_J(M\,|\,x)$ indirectly
by defining the judge's preferences over elements in this set
through a complex (latent) score function $s(k)$,
defined as a smoothed sum of distances between $k$ and three anchor states $\{60, 130, 200 \}$, and an additional linear component.
Together, these determine $\pi(k) \propto p_0(k)\exp(s(k))$.
Binary judge outcomes are sampled from Eq.~\eqref{eq:BT-model} using $s$.
See Appx.~\ref{app:exp-synth} for additional details.

\begin{wrapfigure}{r}{0.41\linewidth}
    \vspace{-1.2em}
    \centering
    \includegraphics[width=0.90\linewidth]{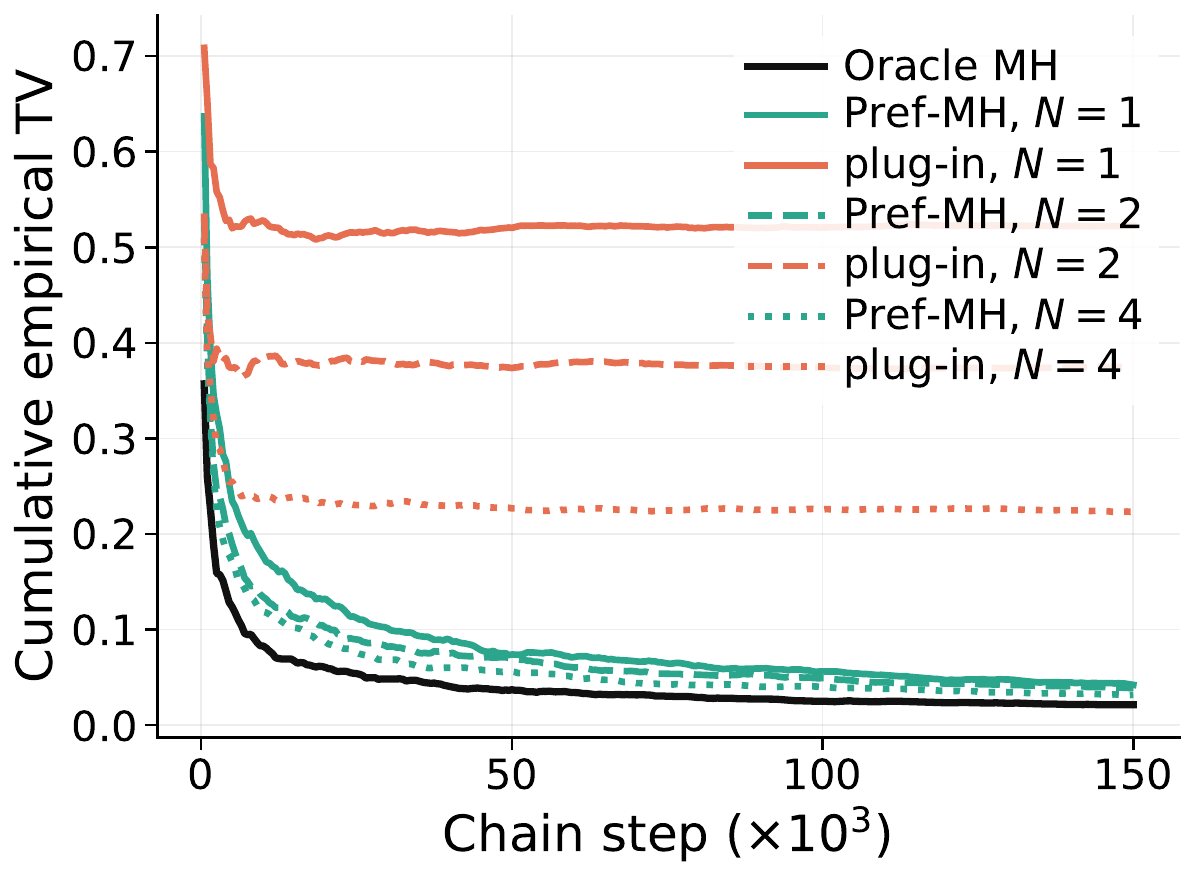}
    \caption{
    \!\!\textbf{Synthetic validation of MH methods.}
    Cumulative empirical total variation distance (TV) between the sampled chain and the true target.
    The distribution is visualized in Fig.~\ref{fig:synthetic-distribution} in the Appx.
    }
    \label{fig:synthetic-convergence}
\end{wrapfigure}

We compare \texttt{Pref-MH} to an invalid \texttt{Plug-in} MH sampler that estimates the
win-rate odds from the observed judge votes and plugs this empirical estimate
into Eq.~\eqref{eq:mh-acceptance-bt}.
By Theorem~\ref{thm:no-fixed-budget-mh}, this plug-in baseline is invalid and is not guaranteed to converge to the target \(\pi\).
We consider $N \in \{1,2,4\}$ for both methods.
As a benchmark, we also compare to an `oracle' MH that has (pointwise)
access to $s$---which is ideal, but infeasible. 
For all methods, we set \(q\) to be a symmetric mixture of a global uniform
proposal over the finite state space and a local nearest-neighbor move. Figure~\ref{fig:synthetic-convergence} shows a clear separation between \texttt{Pref-MH} and the invalid plug-in approach. \texttt{Pref-MH} continues to approach the target distribution as the chain progresses, even with a single judge query per proposal. In contrast, the plug-in chain plateaus at a large error even when using \(N=4\) judge queries.

\subsection{Multi-condition text generation}
\label{exp:multi-cond}

We next consider open-ended text generation under multiple semantic properties.
Llama-3.1-8B-Instruct \citep{grattafiori2024llama} serves as the base model and is prompted to generate the
opening of a story from a short scenario description. The dataset includes $50$ different descriptions, such as \textit{``a weary traveler arriving at a deserted inn just after midnight.''} We
condition on three stylistic properties: $M_1 =$ `elevated language,'
$M_2 =$ `a gothic or dramatic atmosphere,' and $M_3 =$ `frequent dialogue.' The same LLM model is also used as a judge.

We evaluate two variants of our method. \texttt{Pref-MH (sep.)} uses a separate
pairwise judge for each property \(M_i\), while \texttt{Pref-MH (joint)} uses a
single pairwise judge prompted to assess all properties jointly. We compare these
against three baselines. \texttt{Base LLM} directly samples from the base model without prompting for the desired attributes, testing whether MH can steer $p_0$ toward the desired conditional distribution. 
\texttt{Pointwise-MH (sep.)} is an MH sampler that prompts the Llama model to assign pointwise scores, evaluating each property separately, and \texttt{Pointwise-MH (joint)} uses a single
pointwise score intended to assess the three properties jointly.
The pointwise variants provide natural score-based MH baselines, but their scalar scores should
not be interpreted as equivalent to the latent scores underlying the pairwise
comparison model. Thus, even when using the same LLM, pointwise scoring and
pairwise comparison may define different evaluation signals. The proposal kernel
\(q(\cdot \mid x)\) is a suffix-resampling kernel that uses the same Llama model.
See Appx.~\ref{app:exp-stor} for details.

\begin{wrapfigure}{r}{0.47\linewidth}
    \centering
    \vspace{-1em}
    \includegraphics[width=\linewidth]{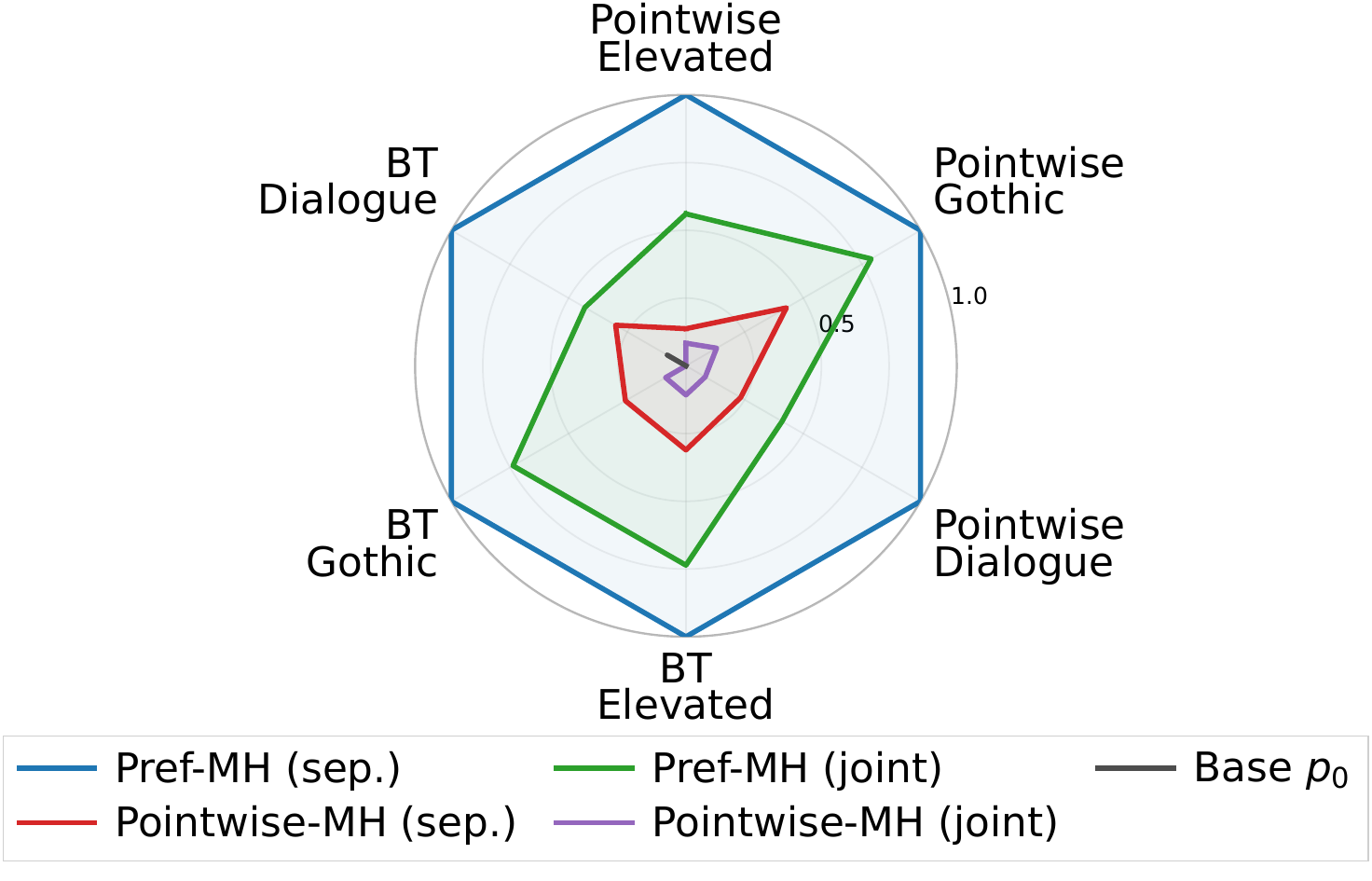}
    \caption{
    \textbf{Passage generation under multiple stylistic properties.}
    Pointwise scores and BT coefficients are normalized across methods for each semantic property; higher is better.
    }
    \label{fig:story_radar}
    \vspace{-1em}
\end{wrapfigure}

We compare methods by evaluating their ability to achieve the three properties.
Since there are no subjective measures for this, 
we use each method's evaluation signal:
(i) pointwise scores for each semantic property, and
(ii) pairwise scores obtained by fitting a BT model to each judge's
binary vote.
Figure~\ref{fig:story_radar} presents the results, where we normalized each evaluation criterion across methods. As shown, \texttt{Pref-MH} with separate judges performs best across all properties under the two metrics considered. The gap between separate and joint \texttt{Pref-MH} suggests that decomposing the conditioning criteria into multiple judgments is advantageous. The inferior performance of the pointwise-MH variants suggests that pairwise comparisons are more discriminative.

\FloatBarrier

\subsection{Image generation in continuous space}
\label{exp:image-generation}

The following stylized experiment demonstrates the applicability of \texttt{Pref-MH} on a completely different modality: conditional image generation. Here,  the state is a continuous latent noise vector. We use a diffusion model (SDXL-Turbo) as the base image generator \citep{sauer2024adversarial, podell2024sdxl}, which maps an input Gaussian noise vector to an image. Since the Gaussian density of each latent state can be evaluated, we run the chain in the noise space and seek noise vectors whose decoded images satisfy the desired properties. We use Qwen3-VL-8B-Instruct \citep{bai2025qwen3} as the VLM judge. See Appx.~\ref{app:exp-img} for details.

We consider multiple visual properties; the image should contain $M_1 =$ `an orange salamander,' $M_2 =$ `a turquoise butterfly,' and $M_3 =$ `a red mushroom.' We compare the chain sampled by \texttt{Pref-MH} to \texttt{Pointwise-MH} and \texttt{base}. The \texttt{Base} SDXL-Turbo model is prompted to generate images with the desired properties. The two MH methods are implemented with the same \texttt{base} method and prompt, but also using a separate VLM for each of the three conditions. \texttt{Pref-MH} uses $N=9$ pairwise votes per judge, while \texttt{Pointwise-MH} uses per-property scores. Both methods use the same proposal $q$, which randomly makes a local move, a wider move, or generates a fresh noise vector.

We evaluate each method's performance by reporting the fraction of images generated in the chain that contain all three desired objects. Object detection is performed using Gemma-3-12B-IT \citep{gemmateam2025gemma3}. Our \texttt{Pref-MH} achieves a success rate of \textbf{63.6\%}, substantially outperforming \texttt{Pointwise-MH} and the base model, which achieve success rates of \textbf{7.4\%} and \textbf{4.5\%}, respectively. At the same time, our approach produces images of comparable visual quality, as measured by the Aesthetic Predictor V2.5 score \citep{aestheticpredictorv25} (higher is better). Specifically, the average aesthetic scores are $6.395$, $6.289$, and $6.287$ for \texttt{Pref-MH}, \texttt{Pointwise-MH}, and the base model, respectively. These quantitative evaluations are consistent with the qualitative results. Figure~\ref{fig:image_generation}
shows representative samples, and full evaluation details are provided in
Appx.~\ref{app:exp-img}.

\begin{figure}[t]
    \centering
    \vspace{-0.8em}
    \includegraphics[width=0.8\linewidth]{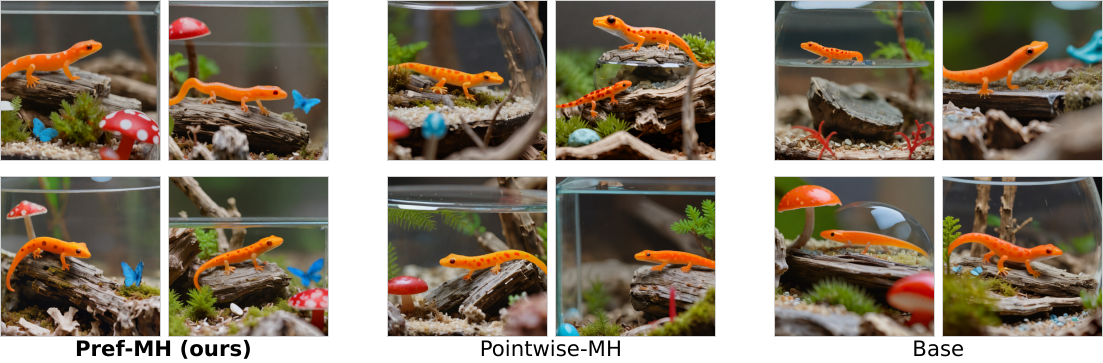}
    \vspace{-0.8em}
    \caption{
    \textbf{Image-generation samples,} targeting three visual conditions:
    an orange salamander, a turquoise butterfly, and a red mushroom.
    Samples correspond to MH steps 300, 1200, 1800, and 2700. Figure~\ref{fig:image_generation_full} presents additional samples. Notice that only \texttt{Pref-MH} achieves the three desired properties.
    }
    \label{fig:image_generation}
    
\end{figure}

\subsection{De novo molecular design from preference feedback}
\label{exp:molecular-design}

We next consider de novo molecular design, the task of generating candidate
molecules with desirable properties for downstream drug-discovery campaigns.
This is a central challenge in computational drug discovery, where the goal is to produce plausible
candidates that can be inspected, compared, and prioritized for experimental or
medicinal-chemistry follow-up. MCMC-based molecular design methods such as MARS \citep{xie2021mars}
instantiate this idea by using multiple hand-crafted cheminformatics scores; specifically, in our experiment we implemented it with proxies for drug-likeness (QED score~\citep{Bickerton2012}) and synthetic accessibility (SA score \citep{Ertl2009}).

While useful, such scores capture only selected aspects of molecular quality and may fail to reflect the full complexity of medicinal-chemistry decision-making.
\citet{choung2023extracting} make this point explicit and developed a surrogate score model, called MolSkill, intended to capture aspects of chemist preference that are not well explained by standard cheminformatics descriptors. The MolSkill score model was trained on collected preference data from expert medicinal chemists. We will use this model as an external evaluator in our experiments.

Inspired by MARS \citep{xie2021mars}, we use a proposal kernel $q$ that, given a current molecule $x$, generates a molecule $x'$ by either randomly deleting a local fragment from $x$ or randomly adding a local fragment from a pre-defined vocabulary. Our \texttt{Pref-MH} uses Qwen3-235B-A22B-Instruct \citep{yang2025qwen3technicalreport} as a pairwise judge, prompted to choose which of two SMILES molecules is more promising for medicinal-chemistry follow-up, with $N=8$. We include \texttt{Pointwise-MH} as a natural baseline, which uses the same Qwen model to provide a pointwise score using the same prompt but now asking the model to rate each molecule from 1 to 10. We query the judge 8 times and average the results, ensuring that the LLM-based baselines use the same compute budget. Additionally, we compare against \texttt{MARS} that uses QED and SA scores, serving as a strong baseline. For this baseline, we implement the original \texttt{MARS} proposal kernel, which learns a model $q$ for proposing local fragment edits.

We run each method for 1000 MCMC steps and evaluate the final state of each chain using the following molecular-design metrics: (1) QED for drug-likeness~\citep{Bickerton2012}; (2) SA for synthetic accessibility~\citep{Ertl2009}, rescaled so that higher values indicate easier synthesis; (3) diversity~\citep{xie2021mars}; and (4) MolSkill~\citep{choung2023extracting}, reported by both mean and median, as our primary external learned proxy for medicinal-chemist follow-up preference. See Appx.~\ref{app:exp-molecular} for implementation details.

\begin{table}[t]
\centering
\small
\setlength{\tabcolsep}{5pt}
\begin{tabular}{@{}lccccc@{}}
\toprule
Method
& QED
& SA
& Div.
& MolSkill mean
& MolSkill median \\
\midrule
\texttt{MARS}
& .773 {\scriptsize$\pm$.002}
& .789 {\scriptsize$\pm$.006}
& .895 {\scriptsize$\pm$.005}
& -0.226 {\scriptsize$\pm$.544}
& -0.157 {\scriptsize$\pm$.464} \\
\texttt{Pointwise-MH}
& .606 {\scriptsize$\pm$.006}
& .761 {\scriptsize$\pm$.004}
& .900 {\scriptsize$\pm$.001}
& -0.098 {\scriptsize$\pm$.261}
& .303 {\scriptsize$\pm$.309} \\
\texttt{Pref-MH}
& .698 {\scriptsize$\pm$.002}
& .781 {\scriptsize$\pm$.001}
& .899 {\scriptsize$\pm$.001}
& -1.432 {\scriptsize$\pm$.398}
& -1.116 {\scriptsize$\pm$.396} \\
\bottomrule
\end{tabular}
\caption{
\textbf{Molecular design from preference feedback.}
Final-state evaluation after 1000 MCMC steps across cheminformatics metrics and MolSkill. Values report the mean and standard error across the independent runs; higher is better for all metrics, except for MolSkill, for which lower is better.}
\label{tab:molecular-design}
\end{table}

Table~\ref{tab:molecular-design} shows that \texttt{Pref-MH} achieves the best MolSkill mean and median, suggesting that pairwise LLM feedback steers the chain toward molecules that are better aligned with medicinal-chemist follow-up preference than either explicit cheminformatics scores or pointwise LLM scores. At the same time, the remaining metrics indicate that this improvement does not come at a major cost to other quality measures: \texttt{Pref-MH} maintains high diversity and a high SA score.

\section{Discussion}

\texttt{Pref-MH} is a flexible, exact MCMC sampler using only pairwise comparisons with an arbitrary proposal kernel over a general state space. Our work makes the parametric BT assumption that there exists an unknown latent score underlying judges' preferences. This assumption is widely used in modern generative modeling, discrete choice modeling, and beyond, but it may not hold exactly in practice when using LLMs, VLMs, or humans as a judge. A promising future direction is to characterize how violations of the BT assumption affect the coherence of the stationary distribution.

More broadly, our experiments highlight the possibility of treating comparative feedback not merely as an evaluation signal, but as a direct interface for defining sampling distributions. This is especially relevant in domains where the desired notion of quality is difficult to reduce to a pointwise score. In our experiments, this role is played by an LLM judge, yet the same mechanism could in principle be used with human experts. Relying on human judges, however, requires the development of efficient versions of \texttt{Pref-MH} that utilize its ability to incorporate clever proposal functions, as otherwise the annotation burden can be infeasible. This opens a broader route for incorporating powerful generative models and expert comparative feedback into high-stakes design tasks.
\clearpage

\bibliographystyle{unsrtnat}
\bibliography{references}

\clearpage

\appendix

\section{Additional Related Work}
\label{app:related-work}

\paragraph{Inference-time guided generation.}
The need to generate content that satisfies specified properties has motivated
methods for user-guided inference. A natural approach, known as best-of-$N$,
generates multiple outputs from a base model and selects the candidate that
maximizes a learned reward, verifier score, or preference-model score
\citep{cobbe2021training,beirami2024theoretical}. While such procedures can
induce a sampling policy, they are typically used to return a single selected
output per query rather than to sample from an explicitly specified conditional
distribution. Furthermore, their reliance on proxy rewards makes guidance
indirect and susceptible to overoptimization risks
\citep{gao2023scaling}. Accommodating multiple desired properties also requires
additional design choices, such as scalarizing, transforming, or combining
multiple objectives or reward models
\citep{roijers2013survey,hayes2022practical,wang2024transforming}. A different class of methods uses controlled decoding to steer the generative
process. Examples include plug-and-play generation with attribute classifiers
\citep{dathathri2020plug}, future-discriminator guidance
\citep{yang2021fudge}, generative-discriminator guidance
\citep{krause2021gedi}, expert/anti-expert decoding
\citep{liu2021dexperts}, energy-based constrained decoding
\citep{qin2022cold}, reward-augmented decoding
\citep{deng2023reward}, and prefix-scorer based controlled decoding
\citep{mudgal2024controlled}. These approaches offer some degree of
distributional control, but they rely on pointwise scores, classifiers, reward
models, or learned prefix scorers for providing guidance. Our approach differs
in three key respects: (i) it explicitly aims to sample from the generative
distribution conditioned on the desired property as defined by a judge, (ii) it supports guidance via
pairwise-comparison judgments, and (iii) it admits rigorous convergence
guarantees.

\paragraph{MCMC from inexact inputs.}
Our approach relates to an established line of research in the sampling
literature. This field concerns MCMC methods for settings in which the ideal MH
acceptance ratio cannot be evaluated exactly, but one still wants to target the
desired distribution exactly. Pseudo-marginal methods achieve this by running MH
on an augmented state space, replacing the intractable unnormalized target
density or likelihood with a nonnegative unbiased estimator
\citep{beaumont2003estimation,andrieu2009pseudo}. Related exchange algorithms
for doubly intractable distributions introduce auxiliary draws so that
parameter-dependent normalizing constants cancel from the acceptance ratio
\citep{murray2006mcmc}. Bernoulli-factory MCMC instead constructs the
accept/reject decision directly from random coins, rather than first estimating
the full acceptance probability
\citep{gonccalves2017barker,vats2022efficient}. We follow a similar yet
distinct philosophy: rather than constructing an unbiased likelihood estimator
or evaluating an intractable pointwise ratio, we construct the accept/reject
decision directly from stochastic pairwise comparisons, in a way that preserves
the target distribution exactly for any fixed comparison budget.

\paragraph{Preference learning and LLM judges.}
Pairwise comparisons have become ubiquitous in the generative modeling
literature. Human preferences elicited through pairwise feedback are routinely
used for model alignment, either through learned reward models
\citep{christiano2017deep,ziegler2019fine,stiennon2020learning,
ouyang2022training,bai2022training}
or directly through preference-optimization objectives
\citep{rafailov2023direct}. Recently, LLMs have also become widely used as
judges for evaluating and comparing model outputs, either through direct
scoring or rubric-based evaluation
\citep{liu2023g,zheng2023judging}
or through pairwise preference-based evaluation and leaderboards
\citep{liu2024aligning,chiang2024chatbot}.
Our approach is different: we use a model's comparative judgments to
\emph{guide generation} through an exact sampling procedure. Conceptually, this
connects to work on self-rewarding, self-evaluation, and self-refinement in
language models
\citep{yuan2024selfrewarding,madaan2023selfrefine,miao2024selfcheck}.
However, rather than relying on prompting, in-context learning, or iterative
self-improvement, our approach steers the generative process via a principled
conditional sampling procedure. It also enables one model to guide another, as
well as guidance along multiple preference dimensions. This represents, to our
knowledge, a novel perspective on model preferences.
\newpage

\section{Multi-judge Pref-MH with Exact $N$-vote Acceptance Algorithm}
\label{app:multi-judge-alg}

\begin{algorithm}[h]
\caption{Multi-judge \texttt{Pref-MH} with exact $N$-vote acceptance}
\label{alg:multi-judge}
\begin{algorithmic}[1]
\Require Initial state $x^{(0)}\in\mathcal X$, proposal kernel $q(\cdot\mid x)$, judges $J_1,\ldots,J_m$, comparisons per judge $N\ge 1$, iterations $T$
\For{$t=0,1,\ldots,T-1$}
    \State Propose $x' \sim q(\cdot\mid x^{(t)})$
    
    \State For $i=1,\ldots,m$, query $J_i$ independently $N$ times on $(x^{(t)},x')$, and obtain $J_i^{(1)},\ldots,J_i^{(N)}.$
    \State Set $K_i= \sum_{l=1}^N J_i^{(l)}$, and compute    
    \[
    x^{(t+1)}
    =
    \begin{cases}
    x', & \text{with probability }
    \min\!\left\{1,\;
    \dfrac{p_0(x')\,q(x^{(t)}\mid x')}{p_0(x^{(t)})\,q(x'\mid x^{(t)})}
    \cdot
    \prod_{i=1}^m \dfrac{K_i}{N-K_i+1}
    \right\}, \\[1.2ex]
    x^{(t)}, & \text{otherwise.}
    \end{cases}
    \]
\EndFor
\State \Return the trajectory $x^{(0)},x^{(1)},\ldots,x^{(T)}$
\end{algorithmic}
\end{algorithm}

\FloatBarrier

\section{Experimental Details and Additional Results}
\label{app:experiments}

This appendix provides the complete experimental setups underlying
the experiments in the main text, together with additional quantitative and
qualitative results. 

\subsection{Synthetic validation}
\label{app:exp-synth}

The synthetic experiment provides a controlled test of the central theoretical
claim in Section~\ref{exp:synthetic}. Here the target distribution is known
exactly, allowing us to verify whether a sampler based only on binary comparison
outcomes converges to the intended stationary distribution. At the same time,
the construction is deliberately nontrivial: the base distribution is
non-uniform, the target contains several separated modes, and the sampler must
account for both the preference-induced tilt and the base-density ratio.

\paragraph{Experimental setup.}
The state space is the finite set $\mathcal X = \{0,1,\ldots,240\}.$

We define the non-uniform base distribution
\[
    p_0(k)
    =
    \frac{
    \exp\left\{
    -\frac{1}{2}\left(\frac{k-68}{56}\right)^2
    \right\}
    }{
    \sum_{\ell=0}^{240}
    \exp\left\{
    -\frac{1}{2}\left(\frac{\ell-68}{56}\right)^2
    \right\}
    }.
\]
To construct the preference-induced tilt, let
$
    g_{a,\tau}(k)
    =
    \exp\left\{
    -\frac{1}{2}\left(\frac{k-a}{\tau}\right)^2
    \right\},
$
and define the latent score
$
    s(k)
    =
    \frac{5}{240}k
    +0.9\,g_{60,13}(k)
    +0.7\,g_{130,21}(k)
    +1.0\,g_{200,13}(k).
$

The resulting target distribution is
\[
    \pi(k)
    =
    \frac{p_0(k)\exp(s(k))}
    {\sum_{\ell=0}^{240}p_0(\ell)\exp(s(\ell))}.
\]
Although this target is available for evaluation, neither \texttt{Pref-MH}
nor the plug-in baseline is given direct access to $s(k)$ or $\pi(k)$ when
making accept/reject decisions.

For a proposed move from $k$ to $k'$, the synthetic judge follows the
Bradley--Terry model
\[
    p_J(k \prec k')
    =
    \sigma\bigl(s(k')-s(k)\bigr),
    \qquad
    \sigma(u)=\frac{1}{1+\exp(-u)}.
\]
Given a comparison budget $N$, we draw
$
    K \sim \mathrm{Binomial}\left(N,\,p_J(k\prec k')\right),
$
where $K$ is the number of votes favoring the proposal.

All methods use the same symmetric proposal. With probability $\gamma=0.88$,
we draw a global proposal uniformly from $\{0,\ldots,240\}$; with probability
$1-\gamma=0.12$, we use a local nearest-neighbor move. For an interior state
$1\leq k\leq239$: $L(k\mid k)=\frac{1}{2},
    L(k-1\mid k)=L(k+1\mid k)=\frac{1}{4}.
$
At the boundaries, probability mass that would leave the state space is kept at
the current state, so that:
$
    L(0\mid0)=\frac{3}{4},L(1\mid0)=\frac{1}{4},
$
and analogously,
$
    L(240\mid240)=\frac{3}{4}, L(239\mid240)=\frac{1}{4}.
$
The full proposal is therefore
$
    q(k'\mid k)
    =
    \gamma\,\mathrm{Unif}\{0,\ldots,240\}(k')
    +(1-\gamma)L(k'\mid k).
$
 Since $q$ is symmetric, the proposal ratio cancels and
\[
    r_0(k,k')=\frac{p_0(k')}{p_0(k)}.
\]

\paragraph{Sampling methods.}
We compare three samplers that share the same proposal kernel. \texttt{Oracle
MH} has direct access to the latent score and accepts according to
\[
    \alpha_{\mathrm{oracle}}(k,k')
    =
    \min\left\{
    1,\,
    \frac{p_0(k')\exp(s(k'))}{p_0(k)\exp(s(k))}
    \right\}.
\]
This sampler is unavailable in the preference-only setting, but provides a
reference for the behavior of a valid chain targeting $\pi$.

\texttt{Pref-MH} instead uses only the $N$ binary judge outcomes. Given
$K\sim\mathrm{Binomial}(N,p_J(k\prec k'))$, it accepts with probability
\[
    \alpha_N(k,k';K)
    =
    \min\left\{
    1,\,
    \frac{p_0(k')}{p_0(k)}
    \frac{K}{N-K+1}
    \right\}.
\]
We evaluate $N\in\{1,2,4\}$.

Finally, \texttt{Plug-in} uses the same $N$ judge votes but replaces
the exact vote-based construction with the empirical win-rate odds. For
$0<K<N$, it forms
\[
    \widehat w(k,k')
    =
    \frac{K/N}{(N-K)/N}
    =
    \frac{K}{N-K},
\]
and accepts according to
\[
    \alpha_{\mathrm{plug\text{-}in}}(k,k';K)
    =
    \min\left\{
    1,\,
    \frac{p_0(k')}{p_0(k)}
    \frac{K}{N-K}
    \right\}.
\]
For $K=0$ and $K=N$, we set the acceptance probability to $0$ and $1$,
respectively. As established by Theorem~\ref{thm:no-fixed-budget-mh}, this
natural plug-in construction is not guaranteed to preserve the desired target.

\paragraph{Evaluation and results.}
All chains are initialized at $k=0$ and run for $150{,}000$ MH steps. Because
our goal is to display convergence from a common nonstationary initialization,
we do not discard burn-in samples in the convergence plot. Every $500$ steps,
we compute the cumulative empirical distribution
\[
    \widehat\pi_t(k)
    =
    \frac{1}{t}\sum_{r=1}^{t}\mathbf 1\{X_r=k\},
\]
and report its total variation distance from the true target,
\[
    \mathrm{TV}(\widehat\pi_t,\pi)
    =
    \frac{1}{2}
    \sum_{k=0}^{240}
    \left|\widehat\pi_t(k)-\pi(k)\right|.
\]
This is the quantity shown in Figure~\ref{fig:synthetic-convergence}. To make
the limiting behavior directly visible, we also run the methods with $N=2$ and
compare their empirical state distributions with the exact target in
Figure~\ref{fig:synthetic-distribution}.

\begin{figure}[t]
    \centering
    \includegraphics[width=0.55\linewidth]{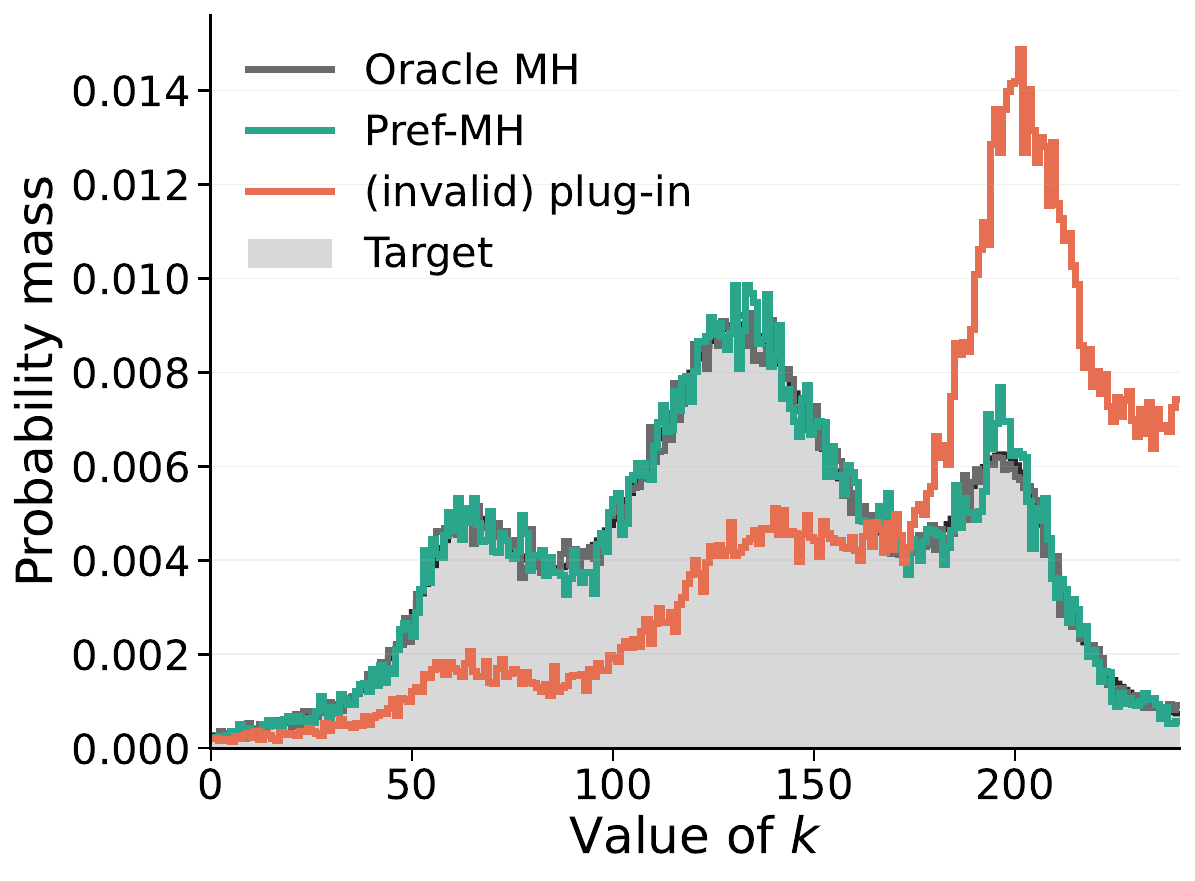}
    \caption{
    \textbf{Synthetic target distribution and empirical samples.}
    We compare the true target distribution over
    $k\in\{0,\ldots,240\}$ with the empirical distributions obtained by
    \texttt{Oracle MH}, \texttt{Pref-MH}, and the invalid plug-in sampler.
    The preference-based methods use $N=2$ judge votes per proposal.
    }
    \label{fig:synthetic-distribution}
\end{figure}

The experiment isolates the effect of the accept/reject construction: the BT
model is exactly satisfied and the target is known. As expected, both
\texttt{Oracle MH} and \texttt{Pref-MH} continue to approach the true target,
with \texttt{Pref-MH} doing so using only a fixed number of binary comparisons
at each transition. The plug-in sampler receives exactly the same comparison
budget, but its TV distance plateaus away from zero. This separation illustrates
the practical consequence of the impossibility result: even under a correctly
specified comparison model, replacing the exact construction by an empirical
odds estimate can change the stationary distribution.

\subsection{Multi-condition text generation}
\label{app:exp-stor}

We next provide the complete setup for the open-ended text experiment in
Section~\ref{exp:multi-cond}. The purpose of this experiment is twofold. First,
it tests whether preference-based MH can steer a language-model distribution
toward several subjective properties at once. Second, it compares two ways of
representing a multi-condition target: a collection of property-specific judges
and a single judge asked to assess the full conjunction.

\paragraph{Experimental setup.}
The task is to generate the opening paragraph of a short story. Each input is a
short scenario description, such as \textit{``a weary traveler arriving at a
deserted inn just after midnight.''} The base model receives the neutral system
prompt
\begin{quote}
\small
\textit{You are a fiction writer. Respond with only the requested opening
paragraph of a story. Do not add explanations, bullet points, titles, or
meta-commentary.}
\end{quote}
The neutral wording is intentional: it prevents the base generator from being
directly instructed to satisfy the hidden evaluation criteria. We use $50$
scenario prompts, limit each opening to at most $192$ new tokens, and sample at
temperature $0.9$.

The target contains three stylistic properties:
$M_1$ is elevated literary language, $M_2$ is a gothic or dramatic atmosphere,
and $M_3$ is frequent direct dialogue. In the separate-judge variants, each
property is handled by its own judge; in the joint variants, one judge is asked
to assess all three properties simultaneously.

We use Llama-3.1-8B-Instruct as the base distribution $p_0$, as the judge, and
as the held-out evaluator, with distinct prompts for the three roles.

The proposal combines global refreshes with local suffix regeneration. With
probability $0.10$, we draw a fresh independent opening from $p_0$. Otherwise,
we sample a cut point uniformly from the full generation limit, retain the
current prefix up to that point, and ask the base model to regenerate the
remaining suffix conditional on both the original scenario and the retained
prefix (inspired by \citep{faria2024quest}). If the selected cut point lies beyond the end of the current response,
the proposal is identical to the current state. Conditional on the selected move
type and cut point, the base-model proposal factor cancels with the corresponding
base-distribution ratio. The acceptance decision therefore depends only on the
judge terms.

\paragraph{Sampling methods.}
We compare five methods. \texttt{Base LLM} draws directly from the neutral base
model and applies no accept/reject correction. It therefore measures what the
unconditioned generator produces without being told the target properties.

The two preference-based methods differ only in how the target is decomposed.
\texttt{Pref-MH (sep.)} uses one pairwise judge for each property. For every
proposal $x'$ from the current state $x$, judge $i$ is queried $N=3$ times and
$K_i$ records the number of votes favoring the proposal. Since the
proposal/base ratio cancels, the acceptance probability is
\[
    \alpha_{\mathrm{sep}}(x,x';K_1,K_2,K_3)
    =
    \min\left\{
    1,\,
    \prod_{i=1}^{3}
    \frac{K_i}{N-K_i+1}
    \right\}.
\]
\texttt{Pref-MH (joint)} instead uses a single pairwise judge prompted to assess
the three properties jointly. If $K$ of the $N=3$ votes favor the proposal, it
accepts according to
\[
    \alpha_{\mathrm{joint}}(x,x';K)
    =
    \min\left\{
    1,\,
    \frac{K}{N-K+1}
    \right\}.
\]

The pointwise baselines mirror the same decomposition while replacing binary
comparisons with absolute ratings. In \texttt{Pointwise-MH (sep.)}, each of the
three judges assigns an integer score from $1$ to $10$, which is normalized to
$[0,1]$. Writing $s_i(x)$ for the normalized score under property $i$, the
proposal is accepted with probability
\[
    \alpha_{\mathrm{abs,sep}}(x,x')
    =
    \min\left\{
    1,\,
    \exp\left(
    \sum_{i=1}^{3}
    \bigl(s_i(x')-s_i(x)\bigr)
    \right)
    \right\}.
\]
\texttt{Pointwise-MH (joint)} uses a single absolute-score judge for the full
conjunction. If $s_{\mathrm{joint}}(x)$ denotes its score, the acceptance
probability is
\[
    \alpha_{\mathrm{abs,joint}}(x,x')
    =
    \min\left\{
    1,\,
    \exp\left(
    s_{\mathrm{joint}}(x')-s_{\mathrm{joint}}(x)
    \right)
    \right\}.
\]

\paragraph{Evaluation and results.}
For each prompt and each MH method, the chain is initialized with a direct sample
from the base model and run for $300$ steps. We discard the first $50$ states as
burn-in and retain the final $50$ states for evaluation, without additional
thinning.

We evaluate the resulting openings using both pointwise and pairwise judgments.
For pointwise evaluation, a held-out prompt asks the evaluator to score elevated
language, gothic atmosphere, direct dialogue, and overall quality on a $1$--$10$
scale. For pairwise evaluation, the evaluator compares outputs from different
methods head-to-head for each individual property. We then fit a
Bradley--Terry model to these binary comparisons and report the resulting BT
coefficients. Figure~\ref{fig:story_radar} normalizes each criterion across
methods for visualization only; Table~\ref{tab:story-results} reports the
underlying values.

\begin{table}[t]
\centering
\caption{
\textbf{Multi-condition text generation results.}
Pointwise scores are normalized to $[0,1]$ and reported as mean $\pm$
standard deviation over evaluated samples. BT columns report Bradley--Terry
coefficients fitted from pairwise comparisons for each criterion. Higher is
better for all metrics.
}
\label{tab:story-results}
\resizebox{\linewidth}{!}{
\begin{tabular}{lcccccc}
\toprule
\multirow{2}{*}{Method}
&
\multicolumn{3}{c}{Pointwise score}
&
\multicolumn{3}{c}{BT coefficient}
\\
\cmidrule(lr){2-4}
\cmidrule(lr){5-7}
&
Elevated
&
Gothic
&
Dialogue
&
Elevated
&
Gothic
&
Dialogue
\\
\midrule
\texttt{Base LLM}
&
$0.801 \pm 0.009$
&
$0.744 \pm 0.111$
&
$0.552 \pm 0.161$
&
$-0.841$
&
$-0.902$
&
$-0.413$
\\
\texttt{Pointwise-MH (sep.)}
&
$0.803 \pm 0.018$
&
$0.767 \pm 0.089$
&
$0.585 \pm 0.161$
&
$-0.361$
&
$-0.474$
&
$-0.161$
\\
\texttt{Pointwise-MH (joint)}
&
$0.802 \pm 0.015$
&
$0.751 \pm 0.107$
&
$0.564 \pm 0.160$
&
$-0.675$
&
$-0.758$
&
$-0.504$
\\
\texttt{Pref-MH (joint)}
&
$0.810 \pm 0.030$
&
$0.786 \pm 0.091$
&
$0.610 \pm 0.174$
&
$0.295$
&
$0.320$
&
$-0.012$
\\
\texttt{Pref-MH (sep.)}
&
$\mathbf{0.817 \pm 0.037}$
&
$\mathbf{0.797 \pm 0.065}$
&
$\mathbf{0.693 \pm 0.147}$
&
$\mathbf{0.703}$
&
$\mathbf{0.755}$
&
$\mathbf{0.639}$
\\
\bottomrule
\end{tabular}
}
\end{table}

The unnormalized results support the same conclusion as the radar plot.
\texttt{Pref-MH (sep.)} performs best under every reported pointwise and
pairwise criterion. The advantage over \texttt{Pref-MH (joint)} is especially
pronounced for dialogue, suggesting that a judge asked to evaluate a complex
conjunction may provide a less discriminative signal than several
property-specific judges. Both preference-based variants also outperform their
pointwise counterparts, indicating that the improvement is not explained by the
use of MH alone, but by the comparative evaluation signal used to define the
accept/reject decisions.

\subsection{Image generation in continuous space}
\label{app:exp-img}

The image experiment demonstrates that \texttt{Pref-MH} is not tied to discrete
sequence spaces. Here the Markov chain evolves in the continuous latent space of
a diffusion model, while all semantic decisions are made by vision-language
judges applied to the decoded images. The experiment also tests the multi-judge
construction in a setting where prompting alone often produces only a subset of
the requested visual attributes.

\paragraph{Experimental setup.}
The state of the chain is a latent noise tensor $z$. We use SDXL-Turbo \citep{sauer2024adversarial, podell2024sdxl} as the
base text-to-image model and decode each latent state using the fixed prompt
\begin{quote}
\small
\textit{A macro photograph of a glass terrarium interior: a bright orange
salamander climbing on a piece of grey driftwood, a small turquoise butterfly
flying in the terrarium, and a red mushroom with white spots growing at the
base near soil.}
\end{quote}
Images are generated at resolution $512\times512$. Since SDXL-Turbo maps
Gaussian latent noise to images, the base distribution over states is the
standard Gaussian prior, $p_0=\mathcal N(0,I)$.

The three semantic properties are
\[
M_1=\text{orange salamander on grey driftwood},
\]
\[
M_2=\text{turquoise butterfly flying in the scene},
\]
\[
M_3=\text{red mushroom with white spots near the soil}.
\]
We use Qwen3-VL-8B-Instruct \citep{bai2025qwen3} as the judge. Each property is assigned a separate
judge prompt, and each judge is instructed to focus only on its own criterion
rather than on the remaining objects or general image aesthetics.

Both MH methods use the same fixed mixture of a local pCN move, a wider pCN move,
and an independent refresh:
\[
q(z'\mid z)
=
0.6\,q_{\mathrm{local}}(z'\mid z)
+
0.3\,q_{\mathrm{wide}}(z'\mid z)
+
0.1\,q_{\mathrm{refresh}}(z'\mid z).
\]
The local and wide components are preconditioned Crank--Nicolson (pCN) proposals
\citep{cotter2013mcmc},
\[
    z'=\sqrt{1-\beta^2}\,z+\beta\xi,
    \qquad \xi\sim\mathcal N(0,I),
\]
with $\beta=0.08$ and $\beta=0.25$, respectively. The refresh component draws
$z'\sim\mathcal N(0,I)$. Each component is reversible with respect to the
Gaussian prior, and a fixed mixture preserves this reversibility. Consequently,
\[
    r_0(z,z')
    =
    \frac{p_0(z')q(z\mid z')}{p_0(z)q(z'\mid z)}
    =1.
\]

\paragraph{Sampling methods.}
The \texttt{Base} method consists of independent latent draws from
$\mathcal N(0,I)$ decoded with the same target-property prompt. It therefore
tests whether direct prompting is sufficient to satisfy the three visual
constraints without any sampling correction.

For \texttt{Pref-MH}, each of the three judges compares the current and proposed
images. Candidate order is randomized before every comparison to reduce
position bias. We query each judge $N=9$ times and let $K_i$ denote the number
of votes favoring the proposal under criterion $i$. Since $r_0(z,z')=1$, the
proposal is accepted with probability
\[
    \alpha_{\mathrm{Pref}}(z,z';K_1,K_2,K_3)
    =
    \min\left\{
    1,\,
    \prod_{i=1}^{3}
    \frac{K_i}{N-K_i+1}
    \right\}.
\]

\texttt{Pointwise-MH} uses the same proposal and the same three criteria, but
replaces each pairwise judge with an absolute VLM score from $0$ to $10$, which
is normalized to $[0,1]$. If $s_i(z)$ denotes the normalized score assigned to
the image decoded from $z$, then
\[
    \alpha_{\mathrm{point}}(z,z')
    =
    \min\left\{
    1,\,
    \exp\left(
    \sum_{i=1}^{3}
    \bigl(s_i(z')-s_i(z)\bigr)
    \right)
    \right\}.
\]

\paragraph{Evaluation and results.}
We run each MH chain for $3000$ steps. The main-text comparison reports samples
from steps $300$, $1200$, $1800$, and $2700$. To show how the methods behave throughout
the trajectory rather than only at a few selected states, we additionally save
samples every $300$ steps. These trajectories are presented in
Figure~\ref{fig:image_generation_full}.

\begin{figure}[!t]
    \centering
    \includegraphics[
        height=0.8\textheight,
        keepaspectratio
    ]{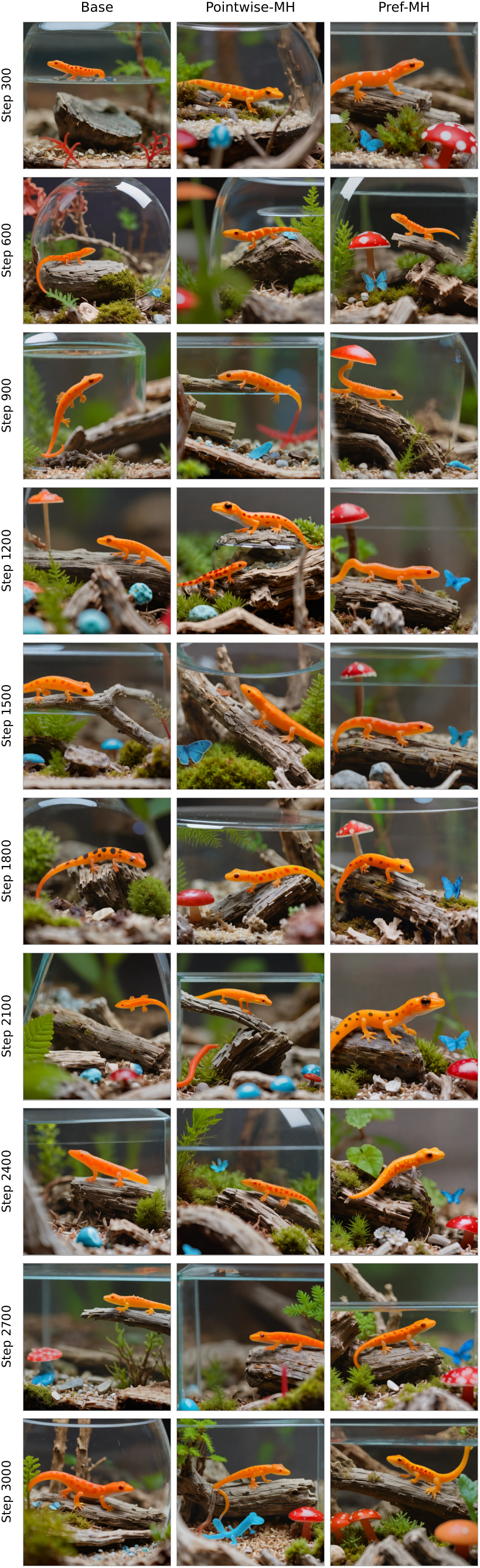}
    \caption{
    \textbf{Additional image-generation trajectory samples.}
    Columns show \texttt{Base}, \texttt{Pointwise-MH}, and \texttt{Pref-MH}
    from left to right; rows show samples every $300$ steps from step $300$ to
    step $3000$.
    }
    \label{fig:image_generation_full}
\end{figure}

\FloatBarrier

In addition to this qualitative trajectory visualization, we perform an
automated quantitative evaluation using models that are separate from the
Qwen3-VL judges used during sampling. To measure simultaneous adherence to the
three desired visual conditions, we use Gemma-3-12B-IT \citep{gemmateam2025gemma3} as a VLM classifier and
report the fraction of images in which all three target objects are detected.
We additionally evaluate overall visual quality using Aesthetic Predictor V2.5 \citep{aestheticpredictorv25}.
Neither evaluator is used by any method during sampling.

The quantitative results reinforce the
qualitative trajectories. \texttt{Pref-MH} satisfies all three desired object
constraints in $63.6\%$ of evaluated images, compared with only $7.4\%$ for
\texttt{Pointwise-MH} and $4.5\%$ for the base model. At the same time, the
average aesthetic scores remain comparable across methods, with
\texttt{Pref-MH} obtaining a slightly higher mean score.  Together with the trajectory samples in
Figure~\ref{fig:image_generation_full}, these results illustrate that the
preference-based construction can operate directly over a continuous latent
space and combine several specialized semantic judges without collapsing them
into a single pointwise score.

\subsection{De novo molecular design}
\label{app:exp-molecular}

We provide additional details for the molecular-design experiment in
Section~\ref{exp:molecular-design}. This experiment tests whether pairwise
feedback can provide a useful guidance signal for molecular generation when the
desired notion of quality is difficult to capture through a small collection of
explicit cheminformatics scores.

\paragraph{Experimental setup.}
We follow the molecular-generation setup of \texttt{MARS} \citep{xie2021mars}. The state
space consists of valid molecules represented as SMILES, and proposals modify
the current molecule through local fragment additions and deletions. The
fragment vocabulary is derived from ChEMBL~\citep{chembl} and contains $1{,}000$
fragments with at most $10$ heavy atoms. All chains are initialized from
$x^{(0)}=\mathtt{CC}$, and generated molecules are restricted to at most $40$
heavy atoms.

For the LLM-guided methods, we use Qwen3-235B-A22B-Instruct-2507
\citep{yang2025qwen3technicalreport} to assess how promising a molecule is for
medicinal-chemistry follow-up. The evaluation considers general
medicinal-chemistry factors such as drug-likeness, synthetic accessibility,
developability, and potentially problematic chemical motifs.

\paragraph{Sampling methods.}
The \texttt{MARS} baseline is guided by QED and synthetic accessibility. Specifically,
we use
\[
s_{\mathrm{MARS}}(x)
=
\frac{1}{2}\left[
\min\{\mathrm{QED}(x),0.7\}
+
\min\{\widetilde{\mathrm{SA}}(x),0.7\}
\right],
\]
where $\widetilde{\mathrm{SA}}(x)=(10-\mathrm{SA}(x))/9$ rescales synthetic
accessibility so that larger values correspond to easier synthesis. A valid
proposal is accepted according to
\[
\alpha_{\mathrm{MARS}}(x,x')
=
\min\left\{
1,\,
r_0(x,x')
\left(
\frac{s_{\mathrm{MARS}}(x')}{s_{\mathrm{MARS}}(x)}
\right)^{\beta}
\right\},
\qquad \beta=30,
\]
where $\beta=30$ is the default value used by \texttt{MARS}
\citep{xie2021mars}.

For \texttt{Pref-MH}, the judge compares the current and proposed molecules
using $N=8$ binary votes. To mitigate positional bias, four comparisons present
the current molecule first and four present the proposed molecule first. If
$K$ votes favor the proposal, it is accepted according to
\[
\alpha_{\mathrm{Pref}}(x,x';K)
=
\min\left\{
1,\,
r_0(x,x')\frac{K}{N-K+1}
\right\}.
\]

\texttt{Pointwise-MH} uses the same LLM and medicinal-chemistry criterion but
asks the judge to assign each molecule an integer rating in
$\{1,\ldots,10\}$. As with \texttt{Pref-MH}, we use $N=8$ judge queries per
proposal and average the resulting ratings. The proposal is then accepted
according to
\[
\alpha_{\mathrm{Point}}(x,x')
=
\min\left\{
1,\,
r_0(x,x')
\left(
\frac{s_{\mathrm{LLM}}(x')}{s_{\mathrm{LLM}}(x)}
\right)^{\beta}
\right\},
\qquad \beta=30.
\]

Following the released \texttt{MARS} implementation, in this experiment we use
the simplified convention $r_0(x,x')=1$ for valid proposals, rather than
explicitly evaluating the corresponding base/proposal factor. We use the same
convention for all three methods, so that the experiment isolates the effect of
the different guidance signals within the \texttt{MARS} molecular-editing
framework.

\paragraph{Evaluation and results.}
We run each method for $1000$ MCMC steps using $200$ parallel chains and repeat
the experiment over $5$ independent random seeds. We evaluate the final state
of each chain and report aggregate results as mean $\pm$ standard error across
the independent runs.

We consider four complementary molecular-design metrics. QED
\citep{Bickerton2012} measures conventional drug-likeness; SA
\citep{Ertl2009} measures synthetic accessibility and is rescaled so that
higher values correspond to easier synthesis; and
diversity is computed from pairwise Tanimoto similarities between Morgan
fingerprints, following \citet{xie2021mars}. Finally, we evaluate the generated
molecules using MolSkill~\citep{choung2023extracting}, a learned proxy derived
from pairwise preferences of professional medicinal chemists. MolSkill is used
only for evaluation and is not available to any of the methods during
sampling. 

As reported in Table~\ref{tab:molecular-design}, \texttt{Pref-MH} achieves the
best MolSkill mean and median while maintaining high synthetic accessibility, and molecular diversity. This suggests that pairwise LLM
feedback captures aspects of medicinal-chemistry follow-up preference that are
not fully represented either by the QED--SA objective used by \texttt{MARS} or by an
elicited pointwise LLM score.

\subsection{Machine translation}
\label{app:exp-mt}

Finally, as an additional real-world evaluation, we consider machine translation and
compare \texttt{Pref-MH} against \texttt{QUEST} \citep{faria2024quest}, a
published MCMC method for quality-aware machine translation. \texttt{QUEST}
uses a learned scalar quality-estimation reward to tilt the base translation
distribution, whereas \texttt{Pref-MH} replaces this numerical acceptance
signal with pairwise judgments. We additionally include a pointwise LLM
baseline which shares the generator, proposal, judge, and per-step query budget
with \texttt{Pref-MH}, but elicits absolute ratings rather than pairwise
preferences.

\paragraph{Experimental setup.}
We evaluate on the WMT23 test sets \citep{kocmi-etal-2023-findings} in both
German$\rightarrow$English and English$\rightarrow$German, using $549$ and
$557$ paired source sentences, respectively. Following~\citet{faria2024quest}, the base distribution $p_0$ is
TowerInstruct-7B-v0.2 \citep{alves2024tower}, sampled at temperature
$\tau=0.8$ with a maximum of $256$ new tokens.

All methods use the suffix-regeneration proposal of
\citet{faria2024quest}. Given a current translation $y$ with
$n=|y|$ tokens, we sample an index $i\sim\mathrm{Unif}\{0,\ldots,n-1\}$,
retain the prefix $y_{<i}$, and regenerate the remaining suffix from the base
model conditioned on the source sentence. If $n'$ denotes the length of the
proposed translation, the corresponding base/proposal factor reduces to
\[
r_0(y,y')
=
\frac{p_0(y')q(y\mid y')}{p_0(y)q(y'\mid y)}
=
\frac{n}{n'}.
\]
Thus, all three samplers use the same proposal mechanism and differ only in how
translation quality enters the accept/reject decision.

For the LLM-guided methods, we use Qwen2.5-32B-Instruct
\citep{qwen2025qwen25technicalreport} as the judge. The judge receives the source sentence,
the candidate translation or translations, and a reference-free quality
estimate, but is never shown the human reference. The quality-estimation (QE)
signal is CometKiwi-XL (wmt23-cometkiwi-da-xl)
\citep{rei-etal-2023-scaling}, the same reference-free QE model used by
\texttt{QUEST}. We also provide this score to the \texttt{Pref-MH} and
\texttt{Pointwise-MH} judges so that all three methods have access to the same
quality-estimation information.

\paragraph{Sampling methods.}
\texttt{QUEST} \citep{faria2024quest} uses the CometKiwi score to define a
reward
$$
r(y,x)
=
\operatorname{logit}\!\left(
\operatorname{clip}\bigl(\mathrm{CometKiwi}(y,x),10^{-3},1-10^{-3}\bigr)
\right)
$$
and targets the corresponding Gibbs-tilted distribution. Its acceptance rule is
\[
\alpha_{\mathrm{QUEST}}(y,y')
=
\min\left\{
1,\,
\frac{n}{n'}
\exp\left(
\frac{r(y',x)-r(y,x)}{\beta}
\right)
\right\}.
\]
The temperature parameter $\beta$ controls the strength of the quality-reward
tilt, with smaller values placing greater weight on differences in the
CometKiwi reward. We evaluate
\[
\beta\in\{0.01,0.02,0.05,0.1,0.2,0.5\},
\]
covering the range considered by \citet{faria2024quest}.

\texttt{Pref-MH} uses $N=5$ independent pairwise judgments for each proposal.
The presentation order of the current and proposed translations is randomized
across queries. If $K$ of the $N$ votes favor the proposal, the acceptance
probability is
\[
\alpha_{\mathrm{Pref}}(y,y';K)
=
\min\left\{
1,\,
\frac{n}{n'}
\frac{K}{N-K+1}
\right\}.
\]

\texttt{Pointwise-MH} uses the same generator, proposal, QE input, judge model,
prompt structure, and $N=5$ judge-query budget. The only change is that the
judge evaluates one translation at a time and returns an integer rating in
$\{1,\ldots,10\}$. The $N$ ratings are averaged to obtain $s(y)$, and the
proposal is accepted according to
\[
\alpha_{\mathrm{Point}}(y,y')
=
\min\left\{
1,\,
\frac{n}{n'}
\exp\left(s(y')-s(y)\right)
\right\}.
\]
The score of the current state is cached, so the pointwise method uses the same
number of judge queries per proposal as \texttt{Pref-MH}.

\paragraph{Evaluation and results.}
Each chain is initialized with an ancestral sample from $p_0$ and run for
$128$ steps, following \citet{faria2024quest}. We use a chain-mean
evaluation protocol: the first $32$ states are discarded as burn-in, after
which $16$ evenly spaced states are retained and their evaluation scores are
averaged for each source sentence.

We report BLEU \citep{papineni-etal-2002-bleu}, chrF \citep{popovic-2015-chrf}, and TER
\citep{snover-etal-2006-study}, computed using \texttt{sacrebleu}
\citep{post-2018-call}, together with XCOMET-XL
\citep{guerreiro2024xcomet}. XCOMET-XL is a reference-based evaluator and is
used only for evaluation; neither it nor the human reference is available
during sampling. Statistical significance is assessed using paired bootstrap
resampling over source sentences with $10{,}000$ resamples, comparing each
baseline against \texttt{Pref-MH} at the $95\%$ level.

\begin{table}[t]
\centering
\scriptsize
\caption{
\textbf{Machine-translation results on WMT23.}
Methods are evaluated using the chain-mean protocol after $32$ burn-in steps
and $16$ retained states per source sentence. Unmarked baseline entries are
significantly worse than \texttt{Pref-MH} at the $95\%$ level;
$\dagger$ denotes a baseline that is significantly better, while
$\ddagger$ denotes no significant difference.
}
\label{tab:mt-results}
\resizebox{\linewidth}{!}{
\begin{tabular}{lcccccccc}
\toprule
\multirow{2}{*}{Method}
& \multicolumn{4}{c}{German $\rightarrow$ English ($n=549$)}
& \multicolumn{4}{c}{English $\rightarrow$ German ($n=557$)}
\\
\cmidrule(lr){2-5}\cmidrule(lr){6-9}
& BLEU $\uparrow$ & chrF $\uparrow$ & TER $\downarrow$ & XCOMET $\uparrow$
& BLEU $\uparrow$ & chrF $\uparrow$ & TER $\downarrow$ & XCOMET $\uparrow$
\\
\midrule
\texttt{QUEST} ($\beta{=}0.01$)
& 35.83 & 63.29 & 51.94 & $.9017^{\dagger}$
& 32.00 & 61.10 & 58.13 & $.8707^{\dagger}$ \\
\texttt{QUEST} ($\beta{=}0.02$)
& 35.99 & 63.13 & 51.82 & $.9025^{\dagger}$
& 32.39 & 61.31 & 57.70 & $.8700^{\dagger}$ \\
\texttt{QUEST} ($\beta{=}0.05$)
& 36.94 & 63.38 & 51.08 & $.8952^{\dagger}$
& 32.30 & 61.20 & 56.37 & $.8654^{\dagger}$ \\
\texttt{QUEST} ($\beta{=}0.1$)
& 37.04 & 63.41 & 49.88 & $.8907^{\ddagger}$
& 32.01 & 60.74 & 56.86 & $.8570^{\dagger}$ \\
\texttt{QUEST} ($\beta{=}0.2$)
& 36.73 & 62.97 & 50.13 & .8828
& 31.64 & 60.46 & 57.12 & .8418 \\
\texttt{QUEST} ($\beta{=}0.5$)
& 36.94 & 62.97 & 49.95 & .8763
& 30.91 & 59.87 & 57.54 & .8278 \\
\midrule
\texttt{Pointwise-MH}
& 37.44 & 63.31 & 49.26 & .8755
& 31.46 & 60.37 & 56.82 & .8252 \\
\midrule
\texttt{Pref-MH}
& \textbf{39.59} & \textbf{65.30} & \textbf{47.57} & .8886
& \textbf{34.09} & \textbf{62.98} & \textbf{54.07} & .8521 \\
\bottomrule
\end{tabular}
}
\end{table}

Table~\ref{tab:mt-results} shows a consistent advantage for
\texttt{Pref-MH} on the reference-based lexical metrics. It obtains the best
BLEU, chrF, and TER in both translation directions and significantly
outperforms every tested \texttt{QUEST} temperature as well as
\texttt{Pointwise-MH} on all three metrics.

XCOMET exhibits a different trade-off. At smaller values of $\beta$, where
\texttt{QUEST} places greater weight on the CometKiwi reward, \texttt{QUEST}
achieves higher XCOMET scores than \texttt{Pref-MH}. As $\beta$ increases,
this advantage disappears and eventually reverses. This is consistent with
XCOMET being more closely aligned with the COMET-family quality-estimation
signal directly optimized by \texttt{QUEST}. Importantly, no tested
\texttt{QUEST} configuration improves on \texttt{Pref-MH} simultaneously on
XCOMET and on any of BLEU, chrF, or TER. Thus, the preference judge appears to
provide a signal that is not reducible to maximizing the scalar QE reward and
whose gains generalize across several independent translation metrics.

The pointwise ablation further isolates the role of comparative feedback.
\texttt{Pref-MH} significantly outperforms \texttt{Pointwise-MH} on all four
evaluation metrics in both translation directions, despite the two methods
sharing the same generator, proposal, QE information, judge model, and
per-step query budget.

\subsection{Compute resources and assets}
\label{app:compute}

\paragraph{Compute resources.}
The synthetic validation experiment was run on CPU using an x86\_64 machine
with two Intel(R) Xeon(R) CPU cores at $2.20$ GHz and $55$ MiB of L3 cache. The
remaining experiments were run on an internal GPU cluster. We used machines
with the following configurations: $224$ Intel(R) Xeon(R) 6746E CPU cores,
$2.0$ TB RAM, and eight NVIDIA H200 GPUs with $143$ GB memory each; and $112$
Intel(R) Xeon(R) 6746E CPU cores, $2.0$ TB RAM, and eight NVIDIA RTX PRO 6000
Blackwell GPUs with $96$ GB memory each.

The machine-translation experiment required approximately $15$ hours per
translation direction, with six H200 GPUs processing disjoint dataset shards in
parallel. The image-generation experiment required approximately $6$ hours on
the GPU cluster, while the multi-condition text experiment required
approximately $10$--$12$ hours on a single GPU. The molecular-design experiment
required approximately $48$ hours on the GPU cluster. No pretrained generative
or judge model was fine-tuned; the reported compute was dominated by model
inference, judge queries, sampling, and evaluation. Additional exploratory runs
were used to refine the experimental setups but are not included in the
reported results.

\paragraph{Existing assets and licenses.}
Our experiments use existing models, datasets, codebases, and evaluation tools.
We use Llama-3.1-8B-Instruct under the Llama 3.1 Community License; WMT23 data
for research use under the WMT23 General Machine Translation task terms;
SDXL-Turbo under the corresponding Stability AI license; Qwen3-VL-8B-Instruct
under the Apache-2.0 license; Qwen3-235B-A22B-Instruct-2507 under the Apache-2.0
license; and COMET/XCOMET under the license specified by the corresponding
Unbabel model release.

For the molecular-design experiment, we build on the official MARS
implementation \citep{xie2021mars}, including its provided fragment vocabulary;
the MARS repository is released under the Creative Commons
Attribution-NonCommercial 4.0 license. We evaluate generated molecules using
the official MolSkill implementation of \citet{choung2023extracting}, which is
released under the MIT License.

\newpage

\section{Mathematical Proofs}
\label{app:proofs}

\subsection{Proof of Theorem~\ref{thm:no-fixed-budget-mh}}
\label{app:no-fixed-budget-mh}

The proof of the following theorem draws connections to other known hardness results for unbiased estimation of odds ratios \citep{degroot1959unbiased, lehmann1998theory, mendo2025estimating}.

\begin{proof}
Fix the pair \((x,x')\) and write
\[
p := p_J(x\prec x'),
\qquad
c := r_0(x,x').
\]
By assumption, \(c\in(0,\infty)\).

Suppose, toward a contradiction, that there exists a fixed-budget procedure
using only the \(N\) binary judge outcomes and additional internal randomness
whose marginal acceptance probability is equal to
\[
\alpha_{\mathrm{BT\text{-}MH}}(p)
=
\min\left\{
1,\,
c\frac{p}{1-p}
\right\}
\]
for every \(p\in(0,1)\).

For each possible outcome vector
\(z=(z_1,\ldots,z_N)\in\{0,1\}^N\), let \(a(z)\in[0,1]\) denote the
probability that the procedure accepts after observing \(z\), where the
probability is taken only over the additional internal randomness of the
procedure. Since the procedure does not have access to the unknown value of
\(p\), the function \(a\) is fixed as a function of the observed binary vector.

Because the judge queries are independent Bernoulli\((p)\) trials, the
probability of observing the vector \(z\) is
\[
\mathbb P_p\big((J^{(1)},\ldots,J^{(N)})=z\big)
=
p^{\sum_{t=1}^N z_t}
(1-p)^{N-\sum_{t=1}^N z_t}.
\]
Therefore, by summing over all possible binary outcome vectors, the marginal
acceptance probability of the procedure is
\[
A_N(p)
=
\sum_{z\in\{0,1\}^N}
a(z)\,
p^{\sum_{t=1}^N z_t}
(1-p)^{N-\sum_{t=1}^N z_t}.
\]
The right-hand side is a polynomial in \(p\) of degree at most \(N\).

By the assumed correctness of the procedure, this polynomial must satisfy
\[
A_N(p)
=
\min\left\{
1,\,
c\frac{p}{1-p}
\right\}
\qquad
\text{for all } p\in(0,1).
\]
Since \(c>0\), the interval \(0<p<1/(1+c)\) is nonempty. On this interval,
\[
c\frac{p}{1-p}<1,
\]
and hence
\[
A_N(p)=c\frac{p}{1-p}.
\]
Equivalently,
\[
(1-p)A_N(p)-cp=0
\]
for every \(p\in(0,1/(1+c))\).

Now define
\[
B(p):=(1-p)A_N(p)-cp.
\]
Since \(A_N(p)\) is a polynomial, \(B(p)\) is also a polynomial. We have just
shown that
\[
B(p)=0
\qquad
\text{for every }p\in(0,1/(1+c)).
\]
Thus \(B\) has infinitely many roots, since every point in this interval is a
root. But a nonzero polynomial of finite degree can have only finitely many
roots. Therefore \(B\) cannot be a nonzero polynomial; it must be the zero
polynomial. Hence
\[
B(p)\equiv 0
\]
as a polynomial identity.

In particular, this identity must also hold when evaluating the polynomial at
\(p=1\). But
\[
B(1)
=
(1-1)A_N(1)-c
=
-c,
\]
so \(B(1)=0\) would imply \(c=0\), contradicting the assumption that
\(c\in(0,\infty)\). Thus no such fixed-budget procedure can exist.
\end{proof}

\subsection{Proof of Theorem~\ref{thm:n-vote-exact}}
\label{app:n-vote-proof}

\begin{proof}
This theorem follows as the single-judge special case of the more general
multi-judge exactness result proved in Appx.~\ref{app:multi-judge}.
Indeed, set \(m=1\) in the multi-judge construction. Then the composite target
\[
\pi_m(x) \propto p_0(x)\prod_{i=1}^m p_{J_i}(M_i\mid x)
\]
reduces to
\[
\pi(x) \propto p_0(x)p_{J}(M\mid x),
\]
and the multi-judge acceptance rule
\[
\alpha_N^{(m)}(x,x';k_1, \ldots, k_m)
=
\min\!\left\{1,\;
r_0(x,x')\prod_{i=1}^m\frac{k_i}{N-k_i+1}
\right\}
\]
from observing the vote-count vector \(k=(k_1,\ldots,k_m)\), reduces exactly to the single-judge rule
\[
\alpha_N(x,x';K)
=
\min\!\left\{1,\;
r_0(x,x')\frac{K}{N-K+1}
\right\}.
\]

The multi-judge proof shows that, for every \(m\geq 1\), the corresponding
marginal acceptance probabilities satisfy
\[
\frac{\bar\alpha_N^{(m)}(x,x')}
     {\bar\alpha_N^{(m)}(x',x)}
=
\frac{\pi_m(x')q(x\mid x')}{\pi_m(x)q(x'\mid x)}.
\]
Taking \(m=1\) gives
\[
\frac{\bar\alpha_N(x,x')}
     {\bar\alpha_N(x',x)}
=
\frac{\pi(x')q(x\mid x')}{\pi(x)q(x'\mid x)}.
\]
Equivalently,
\[
\pi(x)q(x'\mid x)\bar\alpha_N(x,x')
=
\pi(x')q(x\mid x')\bar\alpha_N(x',x),
\]
which is detailed balance with respect to \(\pi\). Hence the single-judge
\(N\)-vote chain is reversible with respect to \(\pi\), and therefore has
\(\pi\) as a stationary distribution.
\end{proof}

\subsection{Convergence of the Pref-MH chain}
\label{app:pref-mh-convergence}

We give sufficient regularity conditions under which the stationary-distribution
result of Theorem~\ref{thm:n-vote-exact} yields convergence from an arbitrary
initialization. Fix a comparison budget $N\ge 1$. Averaging over the random
judge votes and the final accept/reject randomness, define the marginal
acceptance probability
\begin{equation}
\label{eq:alpha-bar}
\bar\alpha_N(x,y)
:=
\mathbb E_{K\sim \mathrm{Binomial}(N,p_J(x\prec y))}
\left[
\min\left\{
1,\,
r_0(x,y)\frac{K}{N-K+1}
\right\}
\right].
\end{equation}
The state sequence generated by Algorithm~\ref{alg:single-judge} is therefore a
time-homogeneous Markov chain with transition kernel
\begin{equation}
\label{eq:pref-mh-kernel}
P_N(x,A)
=
\int_{\mathcal X}
\left[
\bar\alpha_N(x,y)\mathbf 1\{y\in A\}
+
\left(1-\bar\alpha_N(x,y)\right)\mathbf 1\{x\in A\}
\right]
q(dy\mid x),
\end{equation}
for every measurable set $A\subseteq\mathcal X$.

Let $S=\operatorname{supp}(\pi)$.

\begin{assumption}[Proposal regularity]
\label{assump:proposal-regularity}
The proposal kernel $q$ satisfies the following conditions on $S$.

\begin{enumerate}
\item \textbf{$\pi$-irreducibility.}
For every $x\in S$ and every measurable $A\subseteq S$ with $\pi(A)>0$,
there exists $m\ge 1$ such that
\[
q^m(A\mid x)>0,
\]
where $q^m(\cdot\mid x)$ denotes the $m$-step kernel obtained by applying the
proposal kernel successively, without an accept/reject step.

\item \textbf{Compatible support.}
For every $x\in S$,
\[
q\!\left(
\left\{y\in S:r_0(x,y)\notin(0,\infty)\right\}
\,\middle|\,x
\right)
=0.
\]
In words, a proposal drawn from $q(\cdot\mid x)$ has a finite and strictly
positive baseline ratio $r_0(x,y)$ with probability one. In the discrete case,
this simply requires $r_0(x,y)\in(0,\infty)$ for every $y\in S$ satisfying
$q(y\mid x)>0$. When the base density is positive on $S$, this corresponds to
the usual forward--reverse support condition that an admissible forward
proposal also admits the corresponding reverse proposal.

\item \textbf{Target absolute continuity.}
For every $x\in S$,
\[
q(\cdot\mid x)\ll\pi.
\]
That is, any measurable set having zero probability under $\pi$ also has zero
proposal probability from every $x\in S$.
\end{enumerate}
\end{assumption}

The judge requires no additional assumption here. Under the BT
model in Section~\ref{sec:BT-integration}, the latent score is finite-valued,
$s:\mathcal X\to\mathbb R$, and hence
\[
p_J(x\prec y)
=
\sigma(s(y)-s(x))
\in(0,1)
\]
for every $x,y\in\mathcal X$.

We next show that the marginal Pref-MH kernel inherits the required
irreducibility and aperiodicity properties from the proposal.

\paragraph{$\pi$-irreducibility.}
Fix $x\in S$ and a proposed $y$ for which
$r_0(x,y)\in(0,\infty)$, and write $p=p_J(x\prec y)$. Since $p\in(0,1)$,
\begin{align}
\bar\alpha_N(x,y)
&\ge
\mathbb P(K=1)
\min\left\{1,\frac{r_0(x,y)}{N}\right\}
\nonumber\\
&=
Np(1-p)^{N-1}
\min\left\{1,\frac{r_0(x,y)}{N}\right\}
>0.
\label{eq:positive-marginal-acceptance}
\end{align}
Thus every admissible move of $q$ has strictly positive marginal probability
of being accepted by \texttt{Pref-MH}. Consequently, if an $m$-step path of
the proposal kernel reaches a measurable set $A$ with positive probability,
then the event that the same $m$ proposals are generated and accepted also has
positive probability. Hence
\[
q^m(A\mid x)>0
\quad\Longrightarrow\quad
P_N^m(x,A)>0.
\]
By Assumption~\ref{assump:proposal-regularity}(1), $P_N$ is therefore
$\pi$-irreducible on $S$.

\paragraph{Aperiodicity.}
For a proposed move $x\to y$, the event $K=0$ has probability
$(1-p_J(x\prec y))^N>0$, and on this event the proposal is rejected.
Therefore,
\begin{equation}
\label{eq:positive-self-loop}
P_N(x,\{x\})
\ge
\int_S
\left(1-p_J(x\prec y)\right)^N
q(dy\mid x)
>0
\qquad
\text{for every }x\in S.
\end{equation}
Thus the chain has a positive probability of remaining at its current state.
Together with $\pi$-irreducibility, this implies that $P_N$ is aperiodic.

\begin{proof}[Proof of Corollary~\ref{cor:pref-mh-conv}]
Fix $N\ge 1$. By Theorem~\ref{thm:n-vote-exact}, the marginal transition
kernel $P_N$ satisfies detailed balance with respect to $\pi$. Hence $\pi$ is
stationary for $P_N$.

As shown above, Assumption~\ref{assump:proposal-regularity} implies that $P_N$
is $\pi$-irreducible and aperiodic on $S$. The standard general-state-space
convergence theorem (Theorem~1 of \citet{tierney1994markov}),  therefore implies that there exists a measurable set
$G\subseteq S$ with $\pi(G)=1$ such that
\[
\left\|P_N^t(x,\cdot)-\pi\right\|_{\mathrm{TV}}
\longrightarrow 0
\qquad\text{as }t\to\infty
\]
for every $x\in G$
\citep{tierney1994markov,roberts2004general}.

It remains to extend the conclusion from $\pi$-almost every initialization to
every $x\in S$. Since $\pi(G)=1$, Assumption~\ref{assump:proposal-regularity}(3)
gives
\[
q(G\mid x)=1
\qquad
\text{for every }x\in S.
\]
Moreover, by \eqref{eq:positive-marginal-acceptance}, every admissible proposal
has strictly positive marginal acceptance probability. Thus, if $x\notin G$,
then while the chain remains at $x$, each transition has strictly positive
probability of accepting a proposal in $G$. Consequently, the hitting time
\[
\tau_G:=\inf\{t\ge 0:X_t\in G\}
\]
is finite almost surely.

Once the chain enters $G$, it remains in $G$: proposals lie in $G$ almost
surely by Assumption~\ref{assump:proposal-regularity}(3), and rejection leaves
the current state unchanged. Applying the Markov property at $\tau_G$ and the
convergence result on $G$ therefore yields
\[
\left\|P_N^t(x,\cdot)-\pi\right\|_{\mathrm{TV}}
\longrightarrow 0
\qquad\text{for every }x\in S.
\]
In particular, the law of $X_t$ converges to $\pi$ from every initialization in
the support of $\pi$.
\end{proof}

\subsection{Proof of Theorem~\ref{thm:peskun-optimality}}
\label{app:optimality-proof}

\begin{proof}
This theorem follows by taking \(m=1\) in the multi-judge optimality result
proved in Appx.~\ref{app:multi-judge}. In the general
multi-judge setting, a competing exact rule observes \(N\) binary comparisons
from each of \(m\) judges, and the multi-judge \(N\)-vote rule accepts a
proposal after observing the vote-count vector \(k=(k_1,\ldots,k_m)\) with
probability
\[
\alpha_N^{(m)}(x,x';k)
=
\min\!\left\{1,\;
r_0(x,x')\prod_{i=1}^m\frac{k_i}{N-k_i+1}
\right\}.
\]

Setting \(m=1\), the vote-count vector consists of a single count \(K\), and
the acceptance rule becomes
\[
\alpha_N(x,x';K)
=
\min\!\left\{1,\;
r_0(x,x')\frac{K}{N-K+1}
\right\},
\]
which is exactly the single-judge \(N\)-vote rule.

The multi-judge optimality proof shows that, for every exact fixed-budget
competitor using the same proposal kernel and the same number of judge queries,
the competitor's marginal acceptance probability is pointwise no larger than
that of the multi-judge \(N\)-vote rule:
\[
\bar\beta_N^{(m)}(x,x')
\le
\bar\alpha_N^{(m)}(x,x')
\qquad
\text{for every proposed move } x\to x'.
\]
Taking \(m=1\) gives
\[
\bar\beta_N(x,x')
\le
\bar\alpha_N(x,x')
\qquad
\text{for every proposed move } x\to x'.
\]
Since the proposal kernel \(q\) is the same for both procedures, integrating
this pointwise inequality over any set
\(A\subseteq\mathcal X\setminus\{x\}\) gives
\[
P_{\rm comp}(x,A)
\le
P_N^\star(x,A).
\]
Thus, from every current state \(x\), the single-judge \(N\)-vote sampler
assigns at least as much probability as any exact fixed-budget competitor to
moving into any set of states different from \(x\). This is precisely
Peskun--Tierney domination. Therefore, the single-judge \(N\)-vote rule is
optimal among exact fixed-budget acceptance rules using the same proposal
kernel and \(N\) judge queries.
\end{proof}

\subsection{Proofs for multiple judges}
\label{app:multi-judge}

\subsubsection{Multi-judge exactness}

\begin{theorem}[Exactness of the multi-judge $N$-vote acceptance rule]
\label{thm:multi-n-vote-exact}
Under the multi-judge BT parametric assumptions from
Section~\ref{sec:multi-judge}, for every integer \(N\ge 1\), the multi-judge
$N$-vote acceptance rule in Eq.~\eqref{eq:multi-n-vote-main}, combined with a
valid proposal kernel \(q\), defines a Markov transition kernel that satisfies
detailed balance with respect to the composite target \(\pi_m\). Consequently,
the resulting chain is reversible and has \(\pi_m\) as a stationary
distribution.
\end{theorem}

\begin{proof}
To prove exactness of the multi-judge rule, it is enough to verify detailed
balance. In our proposal--accept/reject setting, detailed balance means that
for every pair \(x\neq x'\),
\[
\pi_m(x)\,q(x'\mid x)\,\bar\alpha_N^{(m)}(x,x')
=
\pi_m(x')\,q(x\mid x')\,\bar\alpha_N^{(m)}(x',x),
\]
where \(\bar\alpha_N^{(m)}(x,x')\) denotes the marginal probability of
accepting \(x'\) after it is proposed from \(x\), averaging over the random
judge outcomes from all judges.

Fix a proposal pair \((x,x')\). We first dispose of degenerate proposal-ratio
cases. If \(q(x'\mid x)=0\), then the forward flow
\(\pi_m(x)q(x'\mid x)\bar\alpha_N^{(m)}(x,x')\) is zero. Moreover, if
\(q(x\mid x')>0\), then the reverse baseline ratio \(r_0(x',x)\) is zero, so
the reverse acceptance probability is zero for every realization of the judge
votes; if \(q(x\mid x')=0\), the reverse proposal probability itself is zero.
Thus detailed balance is immediate.

Similarly, if \(q(x'\mid x)>0\) but \(q(x\mid x')=0\), then
\(r_0(x,x')=0\), so the forward acceptance probability is zero for every vote
realization, while the reverse flow is zero because \(q(x\mid x')=0\). Hence
detailed balance is again immediate. The same conclusion holds if
\(p_0(x')=0\), since then \(\pi_m(x')=0\) and \(r_0(x,x')=0\).

Therefore, in the only nontrivial case, both proposal probabilities and both
base probabilities are positive, and hence
\[
0 < r_0(x,x') < \infty.
\]
We now restrict to this case.

Fix a proposal pair \((x,x')\), and write
\[
p_i := p_{J_i}(x\prec x'),
\qquad
c := r_0(x,x').
\]
Thus, when the chain proposes \(x'\) from \(x\), the number of votes from judge
\(J_i\) favoring \(x'\) satisfies
\[
K_i\sim \mathrm{Binomial}(N,p_i),
\]
independently across judges. The realized acceptance probability after
observing the vote-count vector \(k=(k_1,\ldots,k_m)\) is
\[
\alpha_N^{(m)}(x,x';k)
=
\min\!\left\{1,\;
c\prod_{i=1}^m \frac{k_i}{N-k_i+1}
\right\}.
\]
Therefore,
\[
\bar\alpha_N^{(m)}(x,x')
=
\sum_{k_1=0}^N\cdots\sum_{k_m=0}^N
\alpha_N^{(m)}(x,x';k)
\prod_{i=1}^m
\binom{N}{k_i}p_i^{k_i}(1-p_i)^{N-k_i}.
\]

For the reverse move from \(x'\) to \(x\), the comparison probability for
judge \(i\) is \(1-p_i\), and
\[
r_0(x',x)
=
\frac{p_0(x)\,q(x'\mid x)}
     {p_0(x')\,q(x\mid x')}
=
\frac{1}{r_0(x,x')}
=
\frac{1}{c}.
\]
We will show that
\[
\bar\alpha_N^{(m)}(x,x')
=
c\prod_{i=1}^m\frac{p_i}{1-p_i}\,
\bar\alpha_N^{(m)}(x',x).
\]

We first record the identities that will be used in the sum below. For each
\(j_i=1,\ldots,N\), the binomial coefficients satisfy
\[
\binom{N}{j_i-1}
=
\binom{N}{j_i}\frac{j_i}{N-j_i+1}.
\]
Indeed,
\[
\binom{N}{j_i-1}
=
\frac{N!}{(j_i-1)!(N-j_i+1)!}
=
\frac{N!}{j_i!(N-j_i)!}\frac{j_i}{N-j_i+1}
=
\binom{N}{j_i}\frac{j_i}{N-j_i+1}.
\]

Second, the forward and reverse realized acceptance probabilities satisfy
\[
\alpha_N^{(m)}(x,x';N-j+1)
=
c\prod_{i=1}^m\frac{N-j_i+1}{j_i}\,
\alpha_N^{(m)}(x',x;j),
\]
where \(j=(j_1,\ldots,j_m)\) and \(N-j+1\) denotes the vector
\[
(N-j_1+1,\ldots,N-j_m+1).
\]
To see this, note first that
\[
\alpha_N^{(m)}(x,x';N-j+1)
=
\min\!\left\{1,\;
c\prod_{i=1}^m\frac{N-j_i+1}{j_i}
\right\},
\]
because \(N-(N-j_i+1)+1=j_i\) for each coordinate \(i\). For the reverse move,
\(r_0(x',x)=1/c\), and hence
\[
\alpha_N^{(m)}(x',x;j)
=
\min\!\left\{1,\;
\frac{1}{c}\prod_{i=1}^m\frac{j_i}{N-j_i+1}
\right\}.
\]
Now set
\[
A := c\prod_{i=1}^m\frac{N-j_i+1}{j_i}.
\]
Then
\[
\alpha_N^{(m)}(x,x';N-j+1)=\min\{1,A\},
\qquad
\alpha_N^{(m)}(x',x;j)=\min\{1,A^{-1}\}.
\]
Since, for every \(A>0\),
\[
\min\{1,A\}=A\min\{1,A^{-1}\},
\]
we obtain
\[
\alpha_N^{(m)}(x,x';N-j+1)
=
c\prod_{i=1}^m\frac{N-j_i+1}{j_i}\,
\alpha_N^{(m)}(x',x;j).
\]

We now use these identities to compare the marginal forward and reverse
acceptance probabilities. Since
\(\alpha_N^{(m)}(x,x';k)=0\) whenever some \(k_i=0\), we may restrict the sum
to \(k_i\ge 1\) for all \(i\):
\[
\bar\alpha_N^{(m)}(x,x')
=
\sum_{k_1=1}^N\cdots\sum_{k_m=1}^N
\alpha_N^{(m)}(x,x';k)
\prod_{i=1}^m
\binom{N}{k_i}p_i^{k_i}(1-p_i)^{N-k_i}.
\]
Substituting \(k_i=N-j_i+1\) for each \(i=1,\ldots,m\) gives
\[
\bar\alpha_N^{(m)}(x,x')
=
\sum_{j_1=1}^N\cdots\sum_{j_m=1}^N
\alpha_N^{(m)}(x,x';N-j+1)
\prod_{i=1}^m
\binom{N}{j_i-1}
p_i^{N-j_i+1}(1-p_i)^{j_i-1}.
\]

Using the identities above, we have
\[
\begin{aligned}
&\alpha_N^{(m)}(x,x';N-j+1)
\prod_{i=1}^m \binom{N}{j_i-1}
\\
&=
\left(
c\prod_{i=1}^m\frac{N-j_i+1}{j_i}\,
\alpha_N^{(m)}(x',x;j)
\right)
\left(
\prod_{i=1}^m
\binom{N}{j_i}\frac{j_i}{N-j_i+1}
\right)
\\
&=
c\,\alpha_N^{(m)}(x',x;j)
\prod_{i=1}^m\binom{N}{j_i}.
\end{aligned}
\]
Therefore,
\[
\bar\alpha_N^{(m)}(x,x')
=
c
\sum_{j_1=1}^N\cdots\sum_{j_m=1}^N
\alpha_N^{(m)}(x',x;j)
\prod_{i=1}^m
\binom{N}{j_i}
p_i^{N-j_i+1}(1-p_i)^{j_i-1}.
\]
Factoring out \(\prod_{i=1}^m p_i/(1-p_i)\) gives
\[
\bar\alpha_N^{(m)}(x,x')
=
\left(
c\prod_{i=1}^m\frac{p_i}{1-p_i}
\right)
\sum_{j_1=1}^N\cdots\sum_{j_m=1}^N
\alpha_N^{(m)}(x',x;j)
\prod_{i=1}^m
\binom{N}{j_i}
p_i^{N-j_i}(1-p_i)^{j_i}.
\]
Since \(\alpha_N^{(m)}(x',x;j)=0\) whenever some \(j_i=0\), the sum on the
right is exactly
\[
\sum_{j_1=0}^N\cdots\sum_{j_m=0}^N
\alpha_N^{(m)}(x',x;j)
\prod_{i=1}^m
\binom{N}{j_i}
p_i^{N-j_i}(1-p_i)^{j_i}
=
\bar\alpha_N^{(m)}(x',x).
\]
Thus,
\[
\bar\alpha_N^{(m)}(x,x')
=
\left(
c\prod_{i=1}^m\frac{p_i}{1-p_i}\,
\right)
\bar\alpha_N^{(m)}(x',x).
\]

Finally, by the multi-judge Metropolis--Hastings ratio in
\eqref{eq:multi-ratio-main},
\[
c\prod_{i=1}^m\frac{p_i}{1-p_i}
=
\frac{\pi_m(x')\,q(x\mid x')}{\pi_m(x)\,q(x'\mid x)}.
\]
Substituting this into the previous display and multiplying both sides by
\(\pi_m(x)q(x'\mid x)\) gives
\[
\pi_m(x)\,q(x'\mid x)\,\bar\alpha_N^{(m)}(x,x')
=
\pi_m(x')\,q(x\mid x')\,\bar\alpha_N^{(m)}(x',x),
\]
which is detailed balance with respect to \(\pi_m\). Hence the induced
multi-judge Markov chain is reversible with respect to \(\pi_m\), and
\(\pi_m\) is stationary.
\end{proof}

\subsubsection{Convergence of the multi-judge Pref-MH chain}
\label{app:multi-pref-mh-convergence}

\begin{corollary}[Convergence of the multi-judge Pref-MH chain]
\label{cor:multi-pref-mh-conv}
Fix $N\ge 1$. Under Assumption~\ref{assump:proposal-regularity},
with $\pi$ replaced by $\pi_m$, the Markov chain generated by the
multi-judge \texttt{Pref-MH} algorithm converges in total variation to
$\pi_m$ from every initialization in $\operatorname{supp}(\pi_m)$.
\end{corollary}

\begin{proof}
By the multi-judge exactness result, the corresponding Markov chain is
reversible with respect to $\pi_m$, and hence $\pi_m$ is stationary.

The irreducibility and aperiodicity arguments from
Section~\ref{app:pref-mh-convergence} apply directly. Under the multi-judge BT
model, each judge satisfies
\[
p_{J_i}(x\prec y)
=
\sigma\!\bigl(s_i(y)-s_i(x)\bigr)
\in(0,1),
\qquad i=1,\ldots,m.
\]
Thus every admissible proposal has strictly positive marginal probability of
being accepted; for example, the event $K_1=\cdots=K_m=1$ has positive
probability and yields a strictly positive acceptance factor. Hence the chain
inherits $\pi_m$-irreducibility from the proposal kernel.

Moreover, if $K_i=0$ for any judge $i$, then the product acceptance factor is
zero. Since this event has positive probability, the chain has a positive
self-transition probability and is therefore aperiodic.

Finally, Assumption~\ref{assump:proposal-regularity}(3), with $\pi$ replaced by
$\pi_m$, allows the same full-measure-set argument used in the proof of
Corollary~\ref{cor:pref-mh-conv} to extend convergence to every initialization
in $\operatorname{supp}(\pi_m)$. Therefore the chain converges in total
variation to $\pi_m$ from every such initialization.
\end{proof}

\subsubsection{Multi-judge optimality}

\begin{theorem}[Peskun--Tierney optimality of the multi-judge $N$-vote rule]
\label{thm:multi-peskun-optimality}
Fix a proposal kernel $q$ and a comparison budget $N\geq 1$ per judge. Among all
exact acceptance rules satisfying detailed balance with respect to $\pi_m$, that
use the same proposal kernel and only the outcomes of exactly $N$ judge queries
per judge per proposal, the multi-judge $N$-vote rule Peskun--Tierney dominates
every competitor. Equivalently, for every proposed move $x\to x'$, we have
\[
\mathbb P_{\text{multi-judge $N$-vote}}
(\textup{accept }x'\mid x,x'\textup{ proposed})
\geq
\mathbb P_{\text{competitor}}
(\textup{accept }x'\mid x,x'\textup{ proposed}).
\]
\end{theorem}

\begin{proof}
Fix a proposed move $x\to x'$, and write
\[
p_i := p_{J_i}(x\prec x'),
\qquad
c := r_0(x,x').
\]
Thus the reverse move $x'\to x$ has comparison probability $1-p_i$ for judge
$i$, and baseline factor $1/c$.

We first describe the form of any competing rule that uses exactly $N$ binary
queries from each of the $m$ judges. Such a rule may depend on the full
collection of comparison outcomes and may also use additional internal
randomness. After averaging over any such internal randomness, let
$b_k(x,x')\in[0,1]$, for $k=(k_1,\ldots,k_m)$, denote the probability that the
competitor accepts the proposal $x'$ when exactly $k_i$ of the $N$ queries to
judge $J_i$ favor $x'$, for each $i=1,\ldots,m$.

Since the number of votes from judge $J_i$ favoring $x'$ is
$K_i\sim\mathrm{Binomial}(N,p_i)$, independently across judges, the
competitor's marginal acceptance probability can be written as
\[
\bar\beta_N^{(m)}(x,x')
=
\sum_{k_1=0}^N\cdots\sum_{k_m=0}^N
b_k(x,x')
\prod_{i=1}^m
\binom{N}{k_i}p_i^{k_i}(1-p_i)^{N-k_i}.
\]
Similarly, for the reverse move,
\[
\bar\beta_N^{(m)}(x',x)
=
\sum_{j_1=0}^N\cdots\sum_{j_m=0}^N
b_j(x',x)
\prod_{i=1}^m
\binom{N}{j_i}(1-p_i)^{j_i}p_i^{N-j_i}.
\]

Because the competitor is assumed to satisfy detailed balance with respect to
\(\pi_m\), its marginal acceptance probabilities must obey
\[
\pi_m(x)\,q(x'\mid x)\,\bar\beta_N^{(m)}(x,x')
=
\pi_m(x')\,q(x\mid x')\,\bar\beta_N^{(m)}(x',x).
\]
Equivalently,
\[
\bar\beta_N^{(m)}(x,x')
=
\frac{\pi_m(x')q(x\mid x')}{\pi_m(x)q(x'\mid x)}
\bar\beta_N^{(m)}(x',x).
\]
By the multi-judge Bradley--Terry odds identity and the definition of \(r_0\),
\[
\frac{\pi_m(x')q(x\mid x')}{\pi_m(x)q(x'\mid x)}
=
r_0(x,x')\prod_{i=1}^m
\frac{p_{J_i}(x\prec x')}{1-p_{J_i}(x\prec x')}.
\]
Writing \(c=r_0(x,x')\) and \(p_i=p_{J_i}(x\prec x')\), we obtain
\[
\bar\beta_N^{(m)}(x,x')
=
\left(
c\prod_{i=1}^m\frac{p_i}{1-p_i}
\right)\,
\bar\beta_N^{(m)}(x',x).
\]

Substituting the two expressions for the marginal acceptance probabilities gives
\[
\begin{aligned}
&\sum_{k_1=0}^N\cdots\sum_{k_m=0}^N
b_k(x,x')
\prod_{i=1}^m
\binom{N}{k_i}p_i^{k_i}(1-p_i)^{N-k_i}
\\
&=
\left(
c\prod_{i=1}^m\frac{p_i}{1-p_i}\right)
\sum_{j_1=0}^N\cdots\sum_{j_m=0}^N
b_j(x',x)
\prod_{i=1}^m
\binom{N}{j_i}(1-p_i)^{j_i}p_i^{N-j_i}.
\end{aligned}
\]
Since the rule must be exact for the unknown comparison probabilities
\(p_i\in(0,1)\), this identity holds as an identity in
\((p_1,\ldots,p_m)\).

To compare the two sides cleanly, set
\[
t_i := \frac{p_i}{1-p_i},
\qquad i=1,\ldots,m.
\]
Since each \(p_i\in(0,1)\), we may divide both sides by
\(\prod_{i=1}^m(1-p_i)^N\). On the left, each term becomes
\[
\frac{p_i^{k_i}(1-p_i)^{N-k_i}}{(1-p_i)^N}
=
\left(\frac{p_i}{1-p_i}\right)^{k_i}
=
t_i^{k_i}.
\]
On the right, after incorporating the factor
\(\prod_i p_i/(1-p_i)\), each coordinate contributes
\[
\frac{p_i^{N-j_i+1}(1-p_i)^{j_i-1}}{(1-p_i)^N}
=
\left(\frac{p_i}{1-p_i}\right)^{N-j_i+1}
=
t_i^{N-j_i+1}.
\]
Thus we obtain the polynomial identity
\[
\sum_{k_1=0}^N\cdots\sum_{k_m=0}^N
b_k(x,x')
\prod_{i=1}^m
\binom{N}{k_i}t_i^{k_i}
=
c
\sum_{j_1=0}^N\cdots\sum_{j_m=0}^N
b_j(x',x)
\prod_{i=1}^m
\binom{N}{j_i}t_i^{N-j_i+1}.
\]

We first note that this identity forces zero acceptance whenever at least one
judge gives zero votes in favor of the proposed move. Indeed, suppose that
some coordinate \(j_i=0\) on the right-hand side. Then the corresponding term
contains the power \(t_i^{N+1}\). No term on the left-hand side contains a
power of \(t_i\) larger than \(N\). Since the identity holds as a polynomial
identity, the coefficient of every such monomial must be zero. Hence
\[
b_j(x',x)=0
\qquad
\text{whenever some } j_i=0.
\]
Similarly, no term on the right-hand side contains a power \(t_i^0\), so
\[
b_k(x,x')=0
\qquad
\text{whenever some } k_i=0.
\]

We may therefore restrict attention to indices with all coordinates at least
one. The polynomial identity becomes
\[
\sum_{k_1=1}^N\cdots\sum_{k_m=1}^N
b_k(x,x')
\prod_{i=1}^m
\binom{N}{k_i}t_i^{k_i}
=
c
\sum_{j_1=1}^N\cdots\sum_{j_m=1}^N
b_j(x',x)
\prod_{i=1}^m
\binom{N}{j_i}t_i^{N-j_i+1}.
\]

Now reindex the right-hand side coordinatewise by setting
\[
k_i=N-j_i+1,
\qquad i=1,\ldots,m.
\]
Equivalently,
\[
j_i=N-k_i+1.
\]
As \(j_i\) ranges from \(1\) to \(N\), so does \(k_i\). Therefore,
\[
\begin{aligned}
&c
\sum_{j_1=1}^N\cdots\sum_{j_m=1}^N
b_j(x',x)
\prod_{i=1}^m
\binom{N}{j_i}t_i^{N-j_i+1}
\\
&=
c
\sum_{k_1=1}^N\cdots\sum_{k_m=1}^N
b_{N-k+1}(x',x)
\prod_{i=1}^m
\binom{N}{N-k_i+1}t_i^{k_i},
\end{aligned}
\]
where \(N-k+1\) denotes the vector
\[
(N-k_1+1,\ldots,N-k_m+1).
\]
Thus
\[
\sum_{k_1=1}^N\cdots\sum_{k_m=1}^N
b_k(x,x')
\prod_{i=1}^m
\binom{N}{k_i}t_i^{k_i}
=
c
\sum_{k_1=1}^N\cdots\sum_{k_m=1}^N
b_{N-k+1}(x',x)
\prod_{i=1}^m
\binom{N}{N-k_i+1}t_i^{k_i}.
\]

This equality holds for all \(t_i>0\), \(i=1,\ldots,m\). Since both sides are
polynomials in \((t_1,\ldots,t_m)\), their coefficients must agree. Hence, for
every \(k=(k_1,\ldots,k_m)\) with \(k_i\ge 1\) for all \(i\),
\[
b_k(x,x')
\prod_{i=1}^m \binom{N}{k_i}
=
c\,b_{N-k+1}(x',x)
\prod_{i=1}^m \binom{N}{N-k_i+1}.
\]
Because \(b_{N-k+1}(x',x)\le 1\), we get
\[
b_k(x,x')
\le
c\prod_{i=1}^m
\frac{\binom{N}{N-k_i+1}}{\binom{N}{k_i}}.
\]
Using
\[
\binom{N}{N-k_i+1}
=
\binom{N}{k_i-1},
\]
this becomes
\[
b_k(x,x')
\le
c\prod_{i=1}^m
\frac{\binom{N}{k_i-1}}{\binom{N}{k_i}}
=
c\prod_{i=1}^m
\frac{k_i}{N-k_i+1}.
\]
Since \(b_k(x,x')\) is an acceptance probability, we also have
\(b_k(x,x')\le 1\). Therefore, for every vote-count vector
\(k=(k_1,\ldots,k_m)\),
\[
b_k(x,x')
\le
\min\!\left\{1,\;
c\prod_{i=1}^m\frac{k_i}{N-k_i+1}
\right\}.
\]
For vectors with some \(k_i=0\), this also holds because we already showed
\(b_k(x,x')=0\).

But the right-hand side is exactly the realized acceptance probability of the
multi-judge \(N\)-vote rule after observing the vote-count vector \(k\):
\[
\alpha_N^{(m)}(x,x';k)
=
\min\!\left\{1,\;
r_0(x,x')\prod_{i=1}^m\frac{k_i}{N-k_i+1}
\right\}
=
\min\!\left\{1,\;
c\prod_{i=1}^m\frac{k_i}{N-k_i+1}
\right\}.
\]
Thus,
\[
b_k(x,x')\le \alpha_N^{(m)}(x,x';k)
\qquad
\text{for every }k=(k_1,\ldots,k_m).
\]

Averaging over the independent vote counts
\(K_i\sim\mathrm{Binomial}(N,p_i)\), and using that all binomial probabilities
are nonnegative, gives
\[
\begin{aligned}
\bar\beta_N^{(m)}(x,x')
&=
\sum_{k_1=0}^N\cdots\sum_{k_m=0}^N
b_k(x,x')
\prod_{i=1}^m
\binom{N}{k_i}p_i^{k_i}(1-p_i)^{N-k_i}
\\
&\le
\sum_{k_1=0}^N\cdots\sum_{k_m=0}^N
\alpha_N^{(m)}(x,x';k)
\prod_{i=1}^m
\binom{N}{k_i}p_i^{k_i}(1-p_i)^{N-k_i}.
\end{aligned}
\]
The right-hand side is precisely the marginal acceptance probability of the
multi-judge \(N\)-vote rule:
\[
\bar\alpha_N^{(m)}(x,x')
=
\mathbb P_{\text{multi-judge }N\text{-vote}}
(\text{accept }x'\mid x,x'\text{ proposed}).
\]
Therefore,
\[
\mathbb P_{\text{comp}}(\text{accept }x'\mid x,x'\text{ proposed})
=
\bar\beta_N^{(m)}(x,x')
\le
\bar\alpha_N^{(m)}(x,x')
\]
\[
= \mathbb P_{\text{multi-judge }N\text{-vote}}
(\text{accept }x'\mid x,x'\text{ proposed}).
\]

It remains to translate this acceptance-probability dominance into the
Peskun--Tierney comparison of the resulting Markov chains. Let \(P_{N,m}^\star\)
denote the one-step transition rule of the multi-judge \(N\)-vote sampler,
meaning the distribution of the next state after one proposal and accept/reject
step, and let \(P_{\rm comp}\) denote the one-step transition rule of the
competing sampler. Thus, \(P_{N,m}^\star(x,A)\) is the probability that the
multi-judge \(N\)-vote chain moves from the current state \(x\) into a set of
states \(A\) in one step, and \(P_{\rm comp}(x,A)\) is the corresponding
probability for the competitor.

Peskun--Tierney domination compares the probabilities of moving away from the
current state. Therefore, consider any set
\(A\subseteq \mathcal X\setminus\{x\}\), so that \(A\) contains only states
different from the current state. Since both samplers use the same proposal
kernel \(q\), the only difference between them is the probability of accepting
a proposed state. Hence
\[
P_{N,m}^\star(x,A)
=
\int_A \bar\alpha_N^{(m)}(x,y)\,q(dy\mid x),
\]
where \(\bar\alpha_N^{(m)}(x,y)\) is the marginal probability that the
multi-judge \(N\)-vote rule accepts \(y\) after it is proposed from \(x\).
Similarly,
\[
P_{\rm comp}(x,A)
=
\int_A \bar\beta_N^{(m)}(x,y)\,q(dy\mid x),
\]
where \(\bar\beta_N^{(m)}(x,y)\) is the marginal acceptance probability of the
competing rule.

From the pointwise argument above, we have shown that for every proposed move
\(x\to y\),
\[
\bar\beta_N^{(m)}(x,y)\le \bar\alpha_N^{(m)}(x,y).
\]
Because \(q(dy\mid x)\) is the same nonnegative proposal measure for both
samplers, integrating this inequality over \(A\) gives
\[
P_{\rm comp}(x,A)
\le
P_{N,m}^\star(x,A).
\]
Thus, from every current state \(x\), the multi-judge \(N\)-vote sampler
assigns at least as much probability as the competitor to moving into any set
of different states. This is exactly the Peskun--Tierney domination condition.
Therefore, the multi-judge \(N\)-vote rule Peskun--Tierney dominates every
exact fixed-budget competitor.
\end{proof}

\section{Conditional Target Derivation}
\label{app:bayes_mh_derivation}

We derive the product-form target used in the multi-condition setting.
Let $\mathcal{X}$ denote the space of candidate objects, and let
$p_0$ be the base distribution over $\mathcal{X}$, with density or
mass function also denoted by $p_0$. Let
$M_1,\ldots,M_m$ denote the desired conditions, and let
$J_1,\ldots,J_m$ be the associated judges. For each condition $M_i$,
we write $p_{J_i}(M_i\mid x)$ for the judge-induced probability, or
likelihood factor, that $x$ satisfies $M_i$ according to judge $J_i$.

We consider the joint model in which $x\sim p_0$, and the judge-induced
events $M_1,\ldots,M_m$ are conditionally independent given $x$:
\begin{equation}
p_J(M_1,\ldots,M_m\mid x)
=
\prod_{i=1}^m p_{J_i}(M_i\mid x).
\label{eq:multi_cond_indep}
\end{equation}
Equivalently,
\begin{equation}
p_J(M_i\mid x, M_{-i}) = p_{J_i}(M_i\mid x),
\qquad i=1,\ldots,m,
\end{equation}
where $M_{-i}$ denotes all conditions except $M_i$.

By Bayes' rule, the conditional distribution induced by the base model
and the judges is
\begin{align}
p(x\mid M_1,\ldots,M_m)
&=
\frac{
p_0(x)\,p_J(M_1,\ldots,M_m\mid x)
}{
\int_{\mathcal{X}} p_0(u)\,p_J(M_1,\ldots,M_m\mid u)\,d\mu(u)
} \\
&=
\frac{
p_0(x)\prod_{i=1}^m p_{J_i}(M_i\mid x)
}{
\int_{\mathcal{X}} p_0(u)\prod_{i=1}^m p_{J_i}(M_i\mid u)\,d\mu(u)
}.
\end{align}
Thus, defining
\begin{equation}
\tilde{\pi}_m(x)
:=
p_0(x)\prod_{i=1}^m p_{J_i}(M_i\mid x),
\qquad
Z_m
:=
\int_{\mathcal{X}} \tilde{\pi}_m(u)\,d\mu(u),
\end{equation}
we obtain the normalized target
\begin{equation}
\pi_m(x)
:=
\frac{\tilde{\pi}_m(x)}{Z_m}
\propto
p_0(x)\prod_{i=1}^m p_{J_i}(M_i\mid x).
\label{eq:multi_target_appendix}
\end{equation}
When the judge likelihoods satisfy
$p_{J_i}(M_i\mid x)\propto \exp(s_i(x))$, this becomes
\begin{equation}
\pi_m(x)
\propto
p_0(x)\exp\left(\sum_{i=1}^m s_i(x)\right),
\end{equation}
which is the multi-condition target used in the main text.

\end{document}